\documentclass{article}

\usepackage{iclr2027_conference,times}
\usepackage[utf8]{inputenc}
\usepackage{aliases}
\usepackage{algorithm}
\usepackage{algorithmic}
\usepackage{multirow}
\usepackage{graphicx}
\usepackage{subcaption}
\usepackage{dsfont}
\usepackage{tikz}
\usepackage{wrapfig}
\usetikzlibrary{fit}

\usepackage{math_commands}

\title{COMPOTE: Collapse-Resistant Multiple-Hypothesis Learning for Sequential Recommendation}
\author{
Louis Castres Saint Martin$^{1,2}$,
Cl\'ement Calauz\`enes$^{1}$,
Alain Rakotomamonjy$^{1}$,
Patrick Gallinari$^{1,2}$ \\
$^{1}$Criteo AI Lab, Paris, France \\
$^{2}$Sorbonne Universit\'e, CNRS, ISIR, 75005 Paris, France
}

\iclrfinalcopy

\begin{document}

\maketitle
\fancyhead[L]{}
\begin{abstract}
Next-item prediction is inherently ambiguous: a single interaction history often admits
 multiple plausible continuations across distinct catalogue regions.
  Standard sequential recommenders compress these alternatives into a single prediction, 
  conflating user intents. We introduce \textbf{CO}llapse-resistant \textbf{M}ultiple hy\textbf{POT}hesis l\textbf{E}arning (COMPOTE), 
  a multi-hypothesis transformer trained with an annealed objective that transitions 
  from uniform supervision to a winner-takes-all assignment. This mechanism forces 
  hypotheses to specialize into distinct catalogue regions, while inference 
  relies on an efficient, normalizer-free max-similarity rule.
Theoretically, we show that in the population case, the optimal solutionspartition the conditional 
next-item distribution according to a max-entropy quantization principle,
 while pooled scores preserve the Bayes-optimal ranking. 
 We also prove a generalization bound scaling logarithmically with the number of hypotheses.
Across three Amazon Reviews 2023 benchmarks, COMPOTE outperforms strong ID-based, generative,
 and multi-interest baselines using item IDs alone. Matched-capacity ablations
  confirm that performance gains stem directly from the annealed 
  objective, preventing the head collapse observed in conventional
   multi-head architectures, rather than added capacity, ensembling, or routing.
\end{abstract}
\section{Introduction}
\label{sec:introduction}

Recommender systems filter vast product catalogues where their ranking dictates item visibility, directly driving discovery,
engagement, and revenue~\citep{covington2016youtube}. 
As an example, \emph{Sequential recommendation} predicts the next item from ordered interactions, 
closely mirroring real-world deployment where user preferences evolve 
and recent user's interaction  history is most informative~\citep{rendle2010fpmc}.
 There exists essentially two paradigms for sequential recommendation, both of which learn a single conditional model of the next item.
The first family encodes the history with an autoregressive sequence model and scores items against a learned embedding table, moving from recurrent~\citep{hidasi2016gru4rec} and convolutional~\citep{tang2018caser} encoders to causal self-attention~\citep{kang2018sasrec}, backbone of later variants~\citep{sun2019bert4rec,zhou2020s3rec,zhai2024hstu}.
The second family treats retrieval as generation~\citep{tay2022dsi}: TIGER~\citep{rajput2023tiger} quantizes item content embeddings into \emph{semantic identifiers} and decodes the next item's identifier token by token, which helps on rare items at the cost of a content encoder, a learned tokenizer and beam search.
Both families nonetheless learn a \emph{single} conditional model of the next item: one query representation, one score vector, one ranked list.

\paragraph{Next-item prediction is ambiguous.}
An user's interaction  history rarely determines the next item: it points to a \emph{set} of plausible continuations, often 
in unrelated parts of the catalogue, so the next-item distribution is multimodal and a single query vector trained with cross-entropy 
settles between the modes instead of matching any one.
\emph{Multi-interest learning} addresses this by representing each user with several vectors, extracted by 
dynamic routing over capsules in MIND~\citep{li2019mind}, by multi-head attention with a diversity--accuracy aggregation step in ComiRec~\citep{cen2020comirec}, or by activating a few concepts from a shared pool in SINE~\citep{tan2021sine}.
REMI~\citep{xie2023remi} keeps the same architecture and instead adds interest-aware hard negatives and a regularizer against \emph{routing collapse}, where all interests attend to the same historical item.
These methods already rely on a form of winner-takes-all (WTA) training, but it comes with the routing architecture 
and WTA is rarely examined as an objective in its own right.
As a result, the optimal interests are not characterized, and diversity is usually 
encouraged through the architecture or additional penalties.

\paragraph{Our approach: COMPOTE.}
We frame next-item prediction as an inherently ambiguous learning problem and adopt multiple-choice learning (MCL)~\citep{guzmanrivera2012mcl,lee2016smcl,rupprecht2017mhp},
 where a model outputs $M$ candidate hypotheses, but optimizes only the best-performing head.
While this winner-takes-all objective effectively quantizes the conditional next-item distribution~\citep{letzelter2024}, it is prone to sub-optimal local minima where heads that underperform early permanently deactivate.
Annealed MCL~\citep{perera2024annealedMCL} mitigates this initialization sensitivity by annealing a soft-assignment temperature during training.
We adapt this formulation to discrete catalogues with \textbf{CO}llapse-resistant \textbf{M}ultiple hy\textbf{POT}hesis l\textbf{E}arning (COMPOTE; \cref{sec:problem-statement-learning-problem}).
A shared causal transformer feeds $M$ prediction heads, each defining a distinct distribution over the catalogue, which are
optimized via the objective function
$\ell_\tau$ in \cref{eq:loss-free-energy} where $\tau$ is the temperature parameter.
At high temperatures, all heads update uniformly; as the temperature vanishes, the objective smoothly converges to pure winner-takes-all.
For inference and retrieval, we introduce max-similarity, 
a normalizer-free counterpart to the loss's max-pooling operator (\cref{prop:loss-via-pooling}), 
which matches max-pooling accuracy on standard benchmarks while bypassing catalogue-wide softmax operations (\cref{sec:scaling-normalization}).
\paragraph{Contributions.}
Methodologically, we propose COMPOTE, extending annealed multiple-choice learning to sequential recommendation alongside an  efficient, normalizer-free retrieval rule (\cref{sec:problem-statement-learning-problem}). 
On the theoretical front, we first establish a generalization bound scaling logarithmically in the number of hypotheses $M$
 (\cref{cor:wta-generalization-refined}).
We further characterize the optimal hypotheses in the winner-takes-all limit, showing they partition the 
conditional item distribution into maximum-entropy regions (\cref{thm:closed-form}) and that pooled 
hypothesis scores preserve the Bayes ranking up to saturation ties (\cref{prop:emax-ranking}), 
proving head specialization emerges naturally without explicit diversity regularizers. 
  Finally, across three Amazon Reviews 2023 benchmarks, COMPOTE consistently outperforms ID-based, generative,
   and multi-interest sequential baselines using item IDs alone, with capacity-matched ablations attributing these gains directly to the annealed MCL objective rather than model scaling, ensembling, or expert routing (\cref{sec:compote-to-ranking}).

\section{Problem Statement and Architecture}
\label{sec:problem-statement-learning-problem}

We consider a supervised learning setting where a query $\rvx \in \cX$ has an associated
relevant item $\rvy$ in a catalog $\cY$ of $N$ items, and the goal is to learn a model
$g: \cX \to \pS_N$ mapping a query to a ranking of the items, i.e., a permutation of
$\{1, \ldots, N\}$.
Data pairs $(\mathbf{x}, \mathbf{y}) \sim \lP_{\mathbf{x}, \mathbf{y}}$ are drawn i.i.d. from an unknown underlying probability distribution over $\mathcal{X} \times \mathcal{Y}$.
In recommendation, queries are user histories and items are the
entities being recommended (products, movies, etc.).
We define $\Delta^{N\mhyphen1} = \{ \mathbf{p} \in [0, 1]^N : \|\rvp\|_1 = 1 \}$ as the $(N\mhyphen1)$-dimensional probability simplex over $\cY$.

\paragraph{Datasets.}
We instantiate this setting on three categories from the Amazon~2023
review corpus~\citep{hou2024amazon2023}: \textsc{Musical Instruments},
\textsc{Industrial and Scientific}, and \textsc{Video Games}.
A query $\rvx$ is a user's chronologically ordered interaction history,
and the target $\rvy$ is the next reviewed item.
Each category contains approximately $25{,}000$ items and between
$50{,}000$ and $95{,}000$ users (\cref{tab:dataset-stats}).
Because each user contributes few observations relative to the catalog
size, several next items may be plausible given the same history.

\paragraph{Next-item Prediction with Multiple-Hypothesis.}
As motivated in the introduction, the model aims at performing next-item prediction via \emph{multiple hypotheses}. It can be formalized as a function $f: \cX \to (\Delta^{N\mhyphen1})^M$ that returns, for every hypothesis $m \leq M$, a probability distribution $f_m(\rvx) \in \Delta^{N\mhyphen1}$ over the catalog, where $\rvp_{m,i} \triangleq f_m(\rvx)_i$ is
the probability that hypothesis $m$ assigns to item $i$ being relevant to query $\rvx$. The exact architecture of $f$ used in COMPOTE is detailed in \cref{fig:mcl-transformer}.
At inference, the $M$ per-encoder distributions $\rvp = f(\rvx)$ are combined into a single score vector over the catalog, which is then sorted to output the ranking
\begin{align}
    \label{eq:pooled-prediction}
    g = \argsort \circ \omega \circ f
    \quad\text{where}\quad
    \omega(\rvp)= \Big(\max_{m\leq M} \rvp_{m,i}\Big)_{i=1}^N ~~\text{(max-pooling)}.
\end{align}
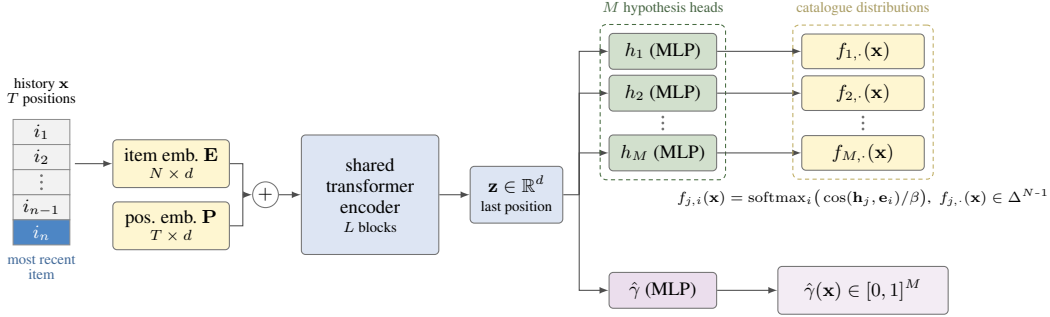
\begin{figure}[!htbp]
    \centering
    \resizebox{\textwidth}{!}{%
    \begin{tikzpicture}[
        x=1cm, y=1cm, font=\small, >=latex,
        blk/.style={draw=black!55, line width=0.5pt, rounded corners=2pt,
                    align=center, inner sep=2.5pt},
        tok/.style={draw=black!45, line width=0.4pt, minimum width=0.95cm,
                    minimum height=0.42cm, inner sep=0pt, fill=black!5,
                    font=\footnotesize},
        tokhl/.style={tok, fill=NavyBlue!75, text=white},
        emb/.style={blk, fill=Goldenrod!25, minimum width=1.95cm,
                    minimum height=0.80cm},
        enc/.style={blk, fill=NavyBlue!12, minimum width=2.30cm,
                    minimum height=2.00cm},
        zst/.style={blk, fill=NavyBlue!12, minimum width=1.55cm,
                    minimum height=0.95cm},
        head/.style={blk, fill=OliveGreen!18, minimum width=1.80cm,
                     minimum height=0.60cm},
        conf/.style={blk, fill=Purple!14, minimum width=1.80cm,
                     minimum height=0.60cm},
        score/.style={blk, fill=Goldenrod!25, minimum width=2.85cm,
                      minimum height=1.25cm},
        gout/.style={blk, fill=Purple!8, minimum width=2.85cm,
                     minimum height=0.80cm},
        wire/.style={line width=0.5pt, draw=black!70},
        ar/.style={wire, ->},
        note/.style={align=center, font=\scriptsize},
        dots/.style={fill=black!55},
    ]

    \node[note] at (0.55, 1.72) {history $\rvx$\\[-1pt]$T$ positions};
    \node[tok]   (t1) at (0.55,  1.05) {$i_1$};
    \node[tok]   (t2) at (0.55,  0.63) {$i_2$};
    \node[tok]   (td) at (0.55,  0.21) {};
    \node[tok]   (t4) at (0.55, -0.21) {$i_{n-1}$};
    \node[tokhl] (t5) at (0.55, -0.63) {$i_{n}$};
    \fill[dots] (0.55, 0.32) circle (0.023)
                (0.55, 0.21) circle (0.023)
                (0.55, 0.10) circle (0.023);
    \node[note, NavyBlue!70!black] at (0.55, -1.22)
        {most recent\\[-1pt]item};
    \node[fit=(t1)(t2)(td)(t4)(t5), inner sep=2pt] (seq) {};

    \node[emb] (embI) at (2.72,  0.52)
        {item emb.\ $\rmE$\\[-2pt]{\scriptsize $N \times d$}};
    \node[emb] (embP) at (2.72, -0.52)
        {pos.\ emb.\ $\rmP$\\[-2pt]{\scriptsize $T \times d$}};
    \draw[ar] (seq.east |- embI.west) -- (embI.west);

    \draw[wire] (embI.east) -- (3.92, 0.52) -- (3.92, -0.52) -- (embP.east);
    \node[circle, draw=black!55, line width=0.5pt, inner sep=1pt,
          font=\footnotesize] (plus) at (4.30, 0) {$+$};
    \draw[ar] (3.92, 0) -- (plus.west);

    \node[enc] (enc) at (6.05, 0)
        {shared \\transformer\\encoder\\[-1pt]{\scriptsize $L$ blocks}};
    \draw[ar] (plus.east) -- (enc.west);

    \node[zst] (z) at (8.50, 0)
        {$\rvz \in \lR^{d}$\\[-1pt]{\scriptsize last position}};
    \draw[ar] (enc.east) -- (z.west);

    \draw[wire] (z.east) -- (9.55, 0);

    \node[head] (h1) at (10.95,  2.40) {$h_1$ (MLP)};
    \node[head] (h2) at (10.95,  1.70) {$h_2$ (MLP)};
    \node[head] (hM) at (10.95,  0.70) {$h_M$ (MLP)};
    \fill[dots] (10.95, 1.30) circle (0.023)
                (10.95, 1.20) circle (0.023)
                (10.95, 1.10) circle (0.023);
    \node[conf] (gam) at (10.95, -1.60) {$\hat{\gamma}$ (MLP)};

    \draw[ar] (9.55, 0) |- (h1.west);
    \draw[ar] (9.55, 0) |- (h2.west);
    \draw[ar] (9.55, 0) |- (hM.west);
    \draw[ar] (9.55, 0) |- (gam.west);

    \node[draw=OliveGreen!70!black, dashed, dash pattern=on 2pt off 1.6pt,
          line width=0.5pt, rounded corners=3pt, inner sep=4pt,
          fit=(h1)(h2)(hM)] (hgrp) {};
    \node[font=\scriptsize, OliveGreen!60!black] at (10.95, 3.12)
        {$M$ hypothesis heads};

    \node[score, minimum width=2.05cm, minimum height=0.58cm] (fj1) at (14.30, 2.40)
        {$f_{1,\cdot}(\rvx)$};
    \node[score, minimum width=2.05cm, minimum height=0.58cm] (fj2) at (14.30, 1.70)
        {$f_{2,\cdot}(\rvx)$};
    \node[score, minimum width=2.05cm, minimum height=0.58cm] (fjM) at (14.30, 0.70)
        {$f_{M,\cdot}(\rvx)$};
    \fill[dots] (14.30, 1.30) circle (0.023)
                (14.30, 1.20) circle (0.023)
                (14.30, 1.10) circle (0.023);

    \draw[ar] (h1.east) -- (fj1.west);
    \draw[ar] (h2.east) -- (fj2.west);
    \draw[ar] (hM.east) -- (fjM.west);

    \node[draw=Goldenrod!70!black, dashed, dash pattern=on 2pt off 1.6pt,
          line width=0.5pt, rounded corners=3pt, inner sep=4pt,
          fit=(fj1)(fj2)(fjM)] (fgrp) {};
    \node[font=\scriptsize, Goldenrod!60!black] at (14.30, 3.12)
        {catalogue distributions};
    \node[note] at (14.30, -0.05)
        {$f_{j,i}(\rvx)=\softmax_i\!\big(\cos(\rvh_j,\rve_i)/\beta\big)$,\;
         $f_{j,\cdot}(\rvx)\in\Delta^{N\mhyphen1}$};

    \node[gout] (gv) at (14.30, -1.60) {$\hat{\gamma}(\rvx) \in [0,1]^{M}$};
    \draw[ar] (gam.east) -- (gv.west);

    \end{tikzpicture}%
    }
    \caption{\textbf{Multiple-choice transformer.}
    An item-id history $\rvx$ is embedded, summed with a learned positional
    embedding, and encoded by a stack of $L$ pre-norm causal transformer blocks
    shared by all heads; the representation $\rvz$ at the last (most recent)
    position is read off.
    Two modules consume $\rvz$: the $M$ hypothesis heads $h_1, \dots, h_M$,
    each producing $\rvh_j = h_j(\rvz) \in \lR^{d}$, and the confidence score
    head $\hat{\gamma}$, whose $j$-th output is the predicted probability that
    head $j$ is the best expert for $\rvx$.
    Each hypothesis is turned into a distribution over the catalogue by
    $f_{j,i}(\rvx) = \softmax_i\!\big(\cos(\rvh_j, \rve_i) / \beta\big)$,
    yielding one score vector $f_{j,\cdot}(\rvx) \in \Delta^{N\mhyphen1}$ per head.
    Everything left of $\rvz$ is a single shared trunk; right of it the $M$
    branches are independent, which is what makes head specialisation possible.}
    \label{fig:mcl-transformer}
\end{figure}

\section{Learning Problem}
\label{sec:learning-problem}

Ideally, to learn $f$ avoiding hypothesis collapse, we would like to use the \emph{winner-takes-all} loss, defined for a given prediction $\rvp = f(\rvx) \in \left(\Delta^{N\mhyphen1}\right)^M$ as
\begin{align}
    \ell^\textsc{wta}(\rvy,\rvp)
    \triangleq
    \min_{m\leq M} \!-\log(\rvp_{m, \rvy})
    ~~\text{with associated risk}~~
    \cL^\textsc{wta}(f, \lP_{\rvx,\rvy})
    = \lE_{\rvx,\rvy}\!\left[\ell^\textsc{wta}(\rvy, f(\rvx))\right].
\end{align}
Optimizing directly $\ell^{\textsc{wta}}$ is usually challenging because of the discontinuity introduced, which leads many people to resort to annealing methods.
Following this path, we use the following loss function defined for temperature $\tau>0$ (to be annealed towards 0) as
\begin{align}
    \label{eq:loss-free-energy}
    \ell_\tau(\rvy, \rvp) = -\tau\log\Big(\sum_{m\leq M} \rvp_{m,\rvy}^\frac{1}{\tau}\Big)
    \quad\text{with associated risk}\quad
    \cL_{\tau}(f, \lP_{\rvx,\rvy})
    = \lE_{\rvx,\rvy}\!\left[\ell_\tau(\rvy, f(\rvx))\right]\,.
\end{align}
As the following result illustrates, annealing on $\ell_\tau$ leads to optimizing $\ell^\textsc{wta}$.
\begin{restatable}[]{proposition}{propTauToWta}
    \label{prop:tau-to-wta}
    For any $(\rvy,\rvp) \in \cY\times \left(\Delta^{N\mhyphen1}\right)^M$,
    $\lim_{\tau\to 0^+} \ell_\tau(\rvy,\rvp) = \ell^\textsc{wta}(\rvy,\rvp)$.
\end{restatable}

Proof can be found in \cref{ssec:prop:proof:tau-to-wta}. This allows to extend by continuity the definition of $\ell_\tau$ to $\tau=0$ and consider $\tau \geq 0$ in the remainder for statements and proofs.

\paragraph{Parallel with Annealed Multiple-Choice Learning (MCL).}
 Methods employing a \emph{winner-takes-all} strategy exist in other domains,
 such as audio source separation, under the framework of \emph{multiple-choice learning} (MCL).
 In fact, our loss $\ell_\tau$ recovers a formulation from this domain, namely
annealed MCL, yet existing literature does not address how to recombine these hypotheses
for inference. Our primary contribution is demonstrating, both theoretically and empirically,
that $\omega$ -- \cref{eq:pooled-prediction} -- is the principled choice for hypothesis aggregation in recommender systems.
Formally, the annealed MCL \citep{perera2024annealedMCL} objective function is defined as
\begin{align}
    \label{eq:awta}
    \ell^{\rm aMCL}_\tau(\rvy, \rvp) = -\sum_{m\leq M} \gamma_m(\rvy, \rvp) \log(\rvp_{m,\rvy})
    \quad \text{with} \quad
    \gamma(\rvy, \rvp) = \softmax\Big(\frac{1}{\tau}\log(\rvp_{\cdot,\rvy})\Big)\,.
\end{align}
and the following propostion establishes the link between the two losses.\\\\

\begin{restatable}[Recovering aMCL]{proposition}{propLossTauToAWta}
    \label{prop:loss-tau-to-awta}
    For any $(\rvy,\rvp) \in \cY\times \left(\Delta^{N\mhyphen1}\right)^M$ and for any $\tau \geq 0$,
    \begin{align}
        \label{eq:loss-tau-to-awta}
        \ell_\tau(\rvy,\rvp) = \ell^{\rm aMCL}_\tau(\rvy,\rvp) - \tau H(\gamma(\rvy,\rvp))
    \end{align}
    where $H(\gamma) = -\sum_{m\leq M} \gamma_m \log \gamma_m$ is the Shannon entropy of the softmax weights $\gamma(\rvy,\rvp)$.
\end{restatable}

Proof can be found in \cref{ssec:prop:proof:loss-tau-to-awta}.
An important detail strengthens the link: \emph{In \cite{perera2024annealedMCL}, during the optimization, the gradient is only backpropagated through $\log \rvp_{m,\rvy}$ and not through $\gamma_m(\rvy, \rvp)$ -- i.e. there is a {\tt stop-gradient} operator on the $\gamma_j(\rvy, \rvp)$ term.} The following result shows the two learning paradigm coincide exactly in practice. Proof can be found in \cref{ssec:prop:proof:awta-gradient}.

\begin{restatable}[Recovering the gradient of aMCL]{corollary}{corAwtaGradient}
    \label{cor:awta-gradient}
    For any $(\rvy,\rvp^\theta) \in \cY\times \left(\Delta^{N\mhyphen1}\right)^M$ ($\theta$ is some parameter), for any $\tau \geq 0$, and $\alpha\in \Delta^{M\mhyphen1}$,
    \begin{align}
        \label{eq:awta-stop-gradient}
        \nabla_\theta \ell_\tau(\rvy,\rvp^\theta) = \left.\nabla_\theta \psi_\tau\left(\rvy, \rvp^\theta, \alpha\right)\right|_{\alpha = \gamma_m(\rvy, \rvp^\theta)}
        ~~~\text{with}~~~
        \psi_\tau(\rvy,\rvp^\theta) = -\!\!\sum_{m\leq M} \!\!\alpha_m \log(\rvp^\theta_{m,\rvy}).
    \end{align}
\end{restatable}

\paragraph{Learning Guarantees.} We have discussed the loss function we use to train the model 
and have established a connection with the annealed MCL loss. We now turn to the theoretical analysis 
of the learning guarantees associated with this loss function. 
While in practice, we learn $f$ by minimizing the empirical risk
$\widehat{\mathcal{L}}_{\tau}(f) = \cL_{\tau}(f, \hat{\lP}_S)$ over a training dataset.
during training, $\tau$ goes to 0. Thus the learning guarantees need to be provided for the limiting loss $\ell^\textsc{wta}$.
Over  a training dataset
$S = \{(\mathbf{x}_i, \mathbf{y}_i)\}_{i=1}^n \sim \lP_{\mathbf{x}, \mathbf{y}}^n$, the empirical risk is denoted $\widehat{\mathcal{L}}^{\textsc{wta}}_S(f) = \cL^\textsc{wta}(f, \hat{\lP}_S)$ where $\hat{\lP}_S$ is the empirical distribution of $S$.
In order to get insights on the impact of the number of hypotheses $M$ on the learning problem,
we are here interested in the generalization gap between the empirical WTA loss and
the expected WTA loss
$\cL^\textsc{wta}(f, {\lP}_S)$
and provide a bound on this gap  in terms of the Rademacher complexity of the hypothesis class $\mathcal{F}$ and
the number of hypotheses $M$. Proof is given in \cref{app:proof-wta-generalization-refined}.




\begin{restatable}[refined log-loss generalization bound]{theorem}{thmWtaGeneralizationRefined}
    \label{cor:wta-generalization-refined}
    Let $\mathcal{F}$ be a hypothesis class of probability-generating heads
    $f_{m}: \mathcal{X} \to \Delta^{N\mhyphen1}$. Assume there exists a constant $\epsilon > 0$
    such that $f_{m,i}(\mathbf{x}) \ge \epsilon$ for all $m \in [M]$, $\mathbf{x} \in \mathcal{X}$, and $i \in [N]$.
    Then, for any $\delta \in (0, 1)$, with probability at least $1 - \delta$ over the sample
    $S \sim P_{\mathbf{x}, \mathbf{y}}^n$, every ensemble $f = (f_{1}, \dots, f_{M}) \in \mathcal{F}^M$ satisfies:
    \begin{equation} \small
        \label{eq:wta-generalization-refined}
        \mathcal{L}^\textsc{wta}(f, {\lP}_S)  \le \widehat{\mathcal{L}}^\textsc{wta}_S(f) + \frac{2}{\epsilon} \max_{m \in [M]} \mathcal{R}_n(\mathcal{F}_{m,\mathbf{y}}) + 2 \log\left(\frac{1}{\epsilon}\right) \sqrt{\frac{2 \ln M}{n}} + \log\left(\frac{1}{\epsilon}\right) \sqrt{\frac{\ln(2/\delta)}{2n}}.
    \end{equation}
     where $\mathcal{R}_n(\mathcal{F}_{m,\mathbf{y}})$ is the expected Rademacher complexity of   $\mathcal{F}_{m,\mathbf{y}} = \{ \mathbf{x} \mapsto f_{m,\mathbf{y}}(\mathbf{x}) : f_{m} \in \mathcal{F} \}$.
\end{restatable}

\textbf{Insights from the bounds:}
\begin{itemize}
    \item The generalization gap is controlled by the Rademacher complexities of the individual heads,
    scaled by $1/\epsilon$. As the minimum predicted probability $\epsilon$ decreases,
     the bounds become looser, reflecting increased uncertainty in predictions.
      Ensuring a lower bound on predicted probabilities, via label smoothing, temperature scaling, or KL regularisation, therefore yields tighter guarantees.
    \item \cref{cor:wta-generalization-refined} shows that the number of hypotheses $M$ impacts the bound only through a $\mathcal{O}(\sqrt{\ln M/n})$ term, on top of the complexity of the best head.
     Adding more hypotheses therefore does not linearly increase the risk of overfitting,
     provided that the individual head complexities are controlled.
     Hence, one can benefit from the diversity of multiple heads without a proportional increase in generalization error.
\end{itemize}

We provide an empirical validation of this theoretical insight in \cref{appendix-section-impact-of-M} in which
we also discuss the practical implications of the saturation phenomenon described in \cref{prop:emax-ranking}
in \cref{sec:compote-to-ranking}.

\paragraph{Learning Stabilization.}
We observe that training $f$ can be stabilized through the joint learning of the confidence score head
via a binary cross-entropy loss over $\hat{\gamma}$ (see \cref{fig:mcl-transformer}).
This finding aligns with the fact that the $\min$ operator is non-smooth and can be challenging to optimize.
While beneficial, this effect is not the primary driver of the overall performance gains (see ablation study in \cref{app:alpha-ablation}).
Consequently, the formalization and technical details of this stabilization term are deferred to the appendix (see \cref{appendix-section-describing-regularizer}).

\section{COMPOTE leads to Optimal Ranking}
\label{sec:compote-to-ranking}

This section explains why the max-pooling aggregation $\omega$ defined in \cref{eq:pooled-prediction} is the natural choice when learning $f$ with $\ell_\tau$. In fact, the max-pooling aggregation $\omega$ can be seen as the limit for a temperature $\tau \to 0$ of some aggregation $\omega_\tau$ (defined below), which cross-entropy is exactly $\ell_\tau$.

\begin{restatable}[$\ell_\tau$ via $\omega^\tau$]{proposition}{propLossViaPooling}
    \label{prop:loss-via-pooling}
    For any $(\rvy,\rvp) \in \cY\times \left(\Delta^{N\mhyphen1}\right)^M$ and for any $\tau \geq 0$,
    \begin{align}
        \label{eq:loss-via-pooling}
        \ell_\tau(\rvy, \rvp) = -\log\omega^\tau_\rvy\left(\rvp\right)
        \quad\text{where}\quad
        \omega^\tau(\rvp) = \Bigg(\Big(\sum_{m\leq M}\rvp_{m,i}^{1/\tau}\Big)^\tau\Bigg)_{i=1}^N
        \quad\text{and}\quad
        \omega^\tau \xrightarrow[\tau\to 0]{} \omega
    \end{align}
\end{restatable}

Proof can be found in \cref{ssec:proof:prop:loss-via-pooling}. Notice that $\omega^\tau(\rvp)$ does not necessarily lie in $\Delta^{N\mhyphen1}$, thus $\cL_\tau(f) = -\lE_{\rvx, \rvy}\!\left[\log\,\omega^\tau_\rvy(f(\rvx))\right]$ can be lower than the Bayes risk $\inf_{q : \cX \to \Delta^{N\mhyphen1}} -\lE_{\rvx, \rvy}\!\left[\log q_\rvy(\rvx)\right]$.
Because $\omega^\tau(\rvp)$ and $\omega(\rvp)$ do not necessarily belong to the simplex, they do not recover the Bayes predictor (they are not a probability).
However, for prediction $\rvp^\star$ that minimizes $\ell^\textsc{wta}$ in expectation, $\omega(\rvp^\star)$ \emph{ranks the items in the same order as the Bayes predictor}, as formalized by the following statement.

\begin{restatable}[$\omega$ induces Bayes ranking]{proposition}{propEmaxRanking}
    \label{prop:emax-ranking}
    Let $\rvq^\star \in \Delta^{N\mhyphen1}$, and denote the optimal prediction $\rvp^\star \in \argmin_{\rvp\in\left(\Delta^{N\mhyphen1}\right)^M} \lE_{\rvy\sim\rvq^\star}[\ell^\textsc{wta}(\rvy, \rvp)]$. Then
    \begin{align}
        \label{eq:emax-ranking}
        \forall\, a,b \in \cY,\qquad
        \rvq^\star_a > \rvq^\star_b \implies
        \omega_a(\rvp^\star) > \omega_b(\rvp^\star)
        \quad\text{or}\quad
        \omega_a(\rvp^\star) = \omega_b(\rvp^\star) = 1\,.
    \end{align}
\end{restatable}

Proof can be found in \cref{ssec:proof:prop:wta-emax-ranking}. The ranking prediction associated to COMPOTE thus preserves 
the order of the Bayes predictor strictly, except in the single \emph{saturation} case $\omega_a(\rvp^\star) = \omega_b(\rvp^\star) = 1$;
 understanding this failure case is informative for the choice of $M$. Intuitively, it happens when (given some $\rvx$) at
  least two items $a$ and $b$ have large probabilities $\rvq_a^\star, \rvq_b^\star \geq 1/M$: 
  for this $\rvx$, one hypothesis $m$ will be such that $\rvp^\star_{m,i} = \indicator{i=a}$ (see \cref{thm:closed-form} for why
   this happens), another hypothesis $m'$ will be such that $\rvp^\star_{m',i} = \indicator{i=b}$. Hence, even though $\rvq_a^\star > \rvq_b^\star$, then $\phi_a = \phi_b = 1$. This advocates for using values for $M$ that are not too large, which is validated experimentally in \cref{fig:num-hypotheses}.

The performance of using the aggregation $\omega$ over the learning done via annealing of $\ell_\tau$ is also validated empirically in the recommendation setting on several Amazon 2023 datasets.

\paragraph{Preprocessing and protocol.}
\label{sec:empirical-evidence}
We evaluate on the three Amazon~2023 subsets introduced in \cref{sec:problem-statement-learning-problem}.
Following prior work~\citep{rajput2023tiger,zhou2020s3rec,zheng2025mtgrec}, we apply the benchmark \emph{5-core} filtering~\citep{amazon2023fivecore}, which iteratively discards users and items with fewer than five interactions, and truncate histories to $50$ items at both training and test time.
We adopt the leave-one-out split of \citet{kang2018sasrec}: the last interaction of each user is the test target, the penultimate one is used for validation and the remainder for training.
We report Hit Rate (H@$K$) and NDCG (N@$K$) for $K\in\{5,10\}$, computed by ranking the entire catalog rather than sampled negatives; since each test instance has a single relevant item, H@$K$ coincides with Recall@$K$.
\paragraph{Baselines.}

We compare against seven ID-based sequential recommenders (Caser~\citep{tang2018caser}, HGN~\citep{ma2019hgn}, GRU4Rec~\citep{hidasi2016gru4rec}, BERT4Rec~\citep{sun2019bert4rec}, SASRec~\citep{kang2018sasrec}, HSTU~\citep{zhai2024hstu} and S3-Rec~\citep{zhou2020s3rec}) and against the generative recommender TIGER~\citep{rajput2023tiger}, whose numbers we take from \citet{zheng2025mtgrec}.
Because COMPOTE predicts several hypotheses per query, we also compare against six multi-interest recommenders that represent a user by several vectors: MIND~\citep{li2019mind}, ComiRec-DR and ComiRec-SA~\citep{cen2020comirec}, SINE~\citep{tan2021sine}, Re4~\citep{zhang2022re4} and REMI~\citep{xie2023remi}, which we train ourselves under the protocol above.
\cref{sec:compared-models} describes each method.

In \cref{tab:results-baselines}, COMPOTE is best in $11$ of the $12$ columns, the exception being H@10 on Scientific where HSTU ($4.29$) and TIGER ($4.22$) edge it out ($4.20$).
Against TIGER, the strongest baseline on NDCG, COMPOTE improves N@5 by $9.7\%$ on Scientific and $9.6\%$ on Game, and ties on Instrument ($2.44 \to 2.45$).
Note that TIGER performs autoregressive generative retrieval over semantic identifiers derived from item text and decodes them with beam search at inference, while COMPOTE sees item identifiers only and ranks the catalogue in a single forward pass: modelling the multimodality of the next-item distribution suffices to beat a content-aware generative recommender.

\paragraph{Multi-interest models trade precision for coverage.}
Multi-interest recommenders are the closest analogue of COMPOTE, yet none reaches the second-best value of any column, and only SINE on Game H@10 ($8.80$) exceeds SASRec.
SINE is within $5\%$ of COMPOTE on that column but $31\%$ below on N@5 ($2.77$ against $4.01$): maximising over interest vectors widens the candidate pool without placing the target near the top, whereas COMPOTE improves both.

\begin{table*}[t]
    \caption{Performance (\%) on three Amazon~2023 subsets (H@$K$: Hit Rate@$K$, N@$K$: NDCG@$K$).
    First-block numbers are from \citet{zheng2025mtgrec}; multi-interest models are trained by us under the same protocol (\cref{sec:compared-models}).
    COMPOTE uses max-similarity aggregation (\cref{eq:max-similarity-aggregation}), equivalent to max-softmax pooling (\cref{sec:scaling-normalization}).
    \textbf{Bold}: best, \underline{underline}: second best.}
    \label{tab:results-baselines}
    \centering
    \footnotesize
    \setlength{\tabcolsep}{4pt}
    \resizebox{\textwidth}{!}{%
    \begin{tabular}{lcccccccccccc}
        \toprule
        & \multicolumn{4}{c}{Instrument} & \multicolumn{4}{c}{Scientific} & \multicolumn{4}{c}{Game} \\
        \cmidrule(lr){2-5} \cmidrule(lr){6-9} \cmidrule(lr){10-13}
        Method & H@5 & H@10 & N@5 & N@10 & H@5 & H@10 & N@5 & N@10 & H@5 & H@10 & N@5 & N@10 \\
        \midrule
        Caser~\citep{tang2018caser} & 2.41 & 3.86 & 1.51 & 1.97 & 1.59 & 2.57 & 1.01 & 1.32 & 3.30 & 5.53 & 2.09 & 2.81 \\
        HGN~\citep{ma2019hgn} & 3.21 & 5.17 & 2.02 & 2.65 & 2.12 & 3.51 & 1.31 & 1.76 & 4.24 & 6.87 & 2.71 & 3.56 \\
        GRU4Rec~\citep{hidasi2016gru4rec} & 3.24 & 5.01 & 2.09 & 2.66 & 2.02 & 3.38 & 1.29 & 1.73 & 4.99 & 7.99 & 3.20 & 4.16 \\
        BERT4Rec~\citep{sun2019bert4rec} & 3.07 & 4.85 & 1.95 & 2.52 & 1.86 & 2.96 & 1.19 & 1.55 & 4.60 & 7.35 & 2.98 & 3.86 \\
        SASRec~\citep{kang2018sasrec} & 3.33 & 5.23 & 2.13 & 2.74 & 2.59 & 4.12 & 1.50 & 1.99 & 5.35 & 8.47 & 3.31 & 4.38 \\
        HSTU~\citep{zhai2024hstu} & 3.43 & \underline{5.77} & 1.91 & 2.71 & \underline{2.71} & \textbf{4.29} & 1.47 & 1.98 & \underline{5.78} & \underline{9.03} & 3.34 & 4.42 \\
        S3-Rec~\citep{zhou2020s3rec} & 3.17 & 4.96 & 1.99 & 2.57 & 2.63 & 4.18 & 1.71 & 2.19 & 4.85 & 7.69 & 3.15 & 4.06 \\
        TIGER~\citep{rajput2023tiger} & \underline{3.70} & 5.64 & \underline{2.44} & \underline{3.06} & 2.64 & \underline{4.22} & \underline{1.75} & \underline{2.26} & 5.59 & 8.68 & \underline{3.66} & \underline{4.67} \\
        \midrule
        MIND~\citep{li2019mind} & 2.71 & 4.42 & 1.65 & 2.20 & 2.11 & 3.29 & 1.27 & 1.65 & 3.22 & 5.20 & 1.99 & 2.65 \\
        ComiRec-DR~\citep{cen2020comirec} & 2.67 & 4.37 & 1.64 & 2.19 & 2.14 & 3.41 & 1.30 & 1.72 & 3.21 & 5.23 & 1.96 & 2.60 \\
        ComiRec-SA~\citep{cen2020comirec} & 2.64 & 4.31 & 1.62 & 2.15 & 2.08 & 3.31 & 1.26 & 1.65 & 3.20 & 5.19 & 1.95 & 2.61 \\
        SINE~\citep{tan2021sine} & 3.08 & 5.20 & 1.52 & 2.24 & 2.36 & 3.83 & 1.21 & 1.68 & 5.26 & 8.80 & 2.77 & 3.87 \\
        Re4~\citep{zhang2022re4} & 2.13 & 3.74 & 1.23 & 1.66 & 1.61 & 2.66 & 0.95 & 1.27 & 3.15 & 5.61 & 1.98 & 2.68 \\
        REMI~\citep{xie2023remi} & 2.73 & 4.34 & 1.71 & 2.22 & 2.20 & 3.86 & 1.21 & 1.70 & 4.14 & 7.42 & 2.29 & 3.30 \\
        \midrule
        COMPOTE (ours) & \textbf{3.75} & \textbf{5.83} & \textbf{2.45} & \textbf{3.10} & \textbf{2.84} & 4.20 & \textbf{1.92} & \textbf{2.35} & \textbf{6.06} & \textbf{9.26} & \textbf{4.01} & \textbf{5.04} \\
        \bottomrule
    \end{tabular}%
    }
\end{table*}

\subsection{Fast Retrieval for COMPOTE}
\label{sec:scaling-normalization}

For large-scale recommender systems, it is desirable to have a model that is compatible with \emph{vector databases} to enable a retrieval scaling sublinearly with $N$. As $f_{m,i}(\rvx) = {\rm softmax}(\cos(h_{m}(x), \rve_i)/\beta)$ (see \cref{fig:mcl-transformer}), there is already a dependency linear in $N$ with the computation of the softmax.
This section shows that, for the purpose of ranking the items for recommendation, the softmax can be dropped at inference -- it is possible to apply $\omega$ directly on the cosine similarities to obtain scores to rank the items.
Formally, it means replace at inference $g = \argsort \circ \omega \circ f$ by
\begin{align}
    \label{eq:max-similarity-aggregation}
    \tilde{g} = \argsort \circ \omega \circ s
    \quad\text{where}\quad
    s_{m,i}(\rvx) = \cos(h_m(\rvx), \rve_i)/\beta)
\end{align}

Due to the following identity: $\log \max_{m\leq M}\left(\softmax_i(\rvs_{m,\cdot})\right) = \max_{m\leq M}\big(\rvs_{m,i} - \log Z_m\big)$ where $Z_m=\sum_{i=1}^N\exp(\rvs_{m,i})$, $g$ and $\tilde{g}$ rank items identically if $Z_m$ is constant across the hypotheses. This turns out to happen in practice as the standard deviation of $Z_m$ across hypotheses is largely lower that its value (\cref{fig:z-m-repartition}).
Experiments confirm this as using $\tilde{g}$ rather than $g$ comes at no loss in ranking accuracy (\cref{tab:aggregation}).

This equivalence is a key deployment result: the model requires neither a catalog-wide softmax nor its normalizer, making the model directly compatible with standard vector databases. For this reason, we report most of the results with $\tilde{g}$ (ex: \cref{tab:results-ablation})

\begin{table}[t]
    \caption{Effect of the inference aggregator on COMPOTE performance (\%) on Game, for the same trained model with $M=3$ heads and the same checkpoint.
    ``Normalizer'' indicates whether ranking requires the per-head catalogue partition function $Z_m$.
    \textbf{Bold}: best.}
    \label{tab:aggregation}
    \centering
    \small
    \setlength{\tabcolsep}{5pt}
    \begin{tabular}{llcccc}
        \toprule
        Aggregator & Normalizer & H@5 & H@10 & N@5 & N@10 \\
        \midrule
        $g$ -- "max softmax" & required & 6.04 & \textbf{9.27} & \textbf{4.01} & \textbf{5.05} \\
        $\tilde{g}$ -- "max cosine similarity" & not required & \textbf{6.06} & 9.26 & \textbf{4.01} & 5.04 \\
        \bottomrule
    \end{tabular}
\end{table}

\subsection{Ablations: Where Do the Gains Come From?}
\label{sec:ablations}

COMPOTE could outperform a single-head model for three distinct reasons, additional capacity, naive ensembling, or mixture-of-experts (MoE) routing. We isolate each with a matched-capacity variant (\cref{tab:results-ablation}).
COMPOTE outperforms all three on every dataset and metric.\\
\emph{Not a bigger model:} MCT-SASRec spends the parameter budget of COMPOTE on a single hypothesis with a deeper prediction head.
It is the strongest of the three variants and still trails COMPOTE everywhere (N@5 of $2.34$ against $2.45$ on Instrument, $1.71$ against $1.92$ on Scientific, $3.78$ against $4.01$ on Game). At equal capacity, the multi-hypothesis parameterization beats a deeper single-hypothesis head.\\
\emph{Not an ensemble:} MCT-EnsCE keeps the architecture of COMPOTE but trains every head on every example with equal weight, without confidence supervision or annealing.
It minimizes $\ell^{\rm ens}(\rvy,f(\rvx)) = -\frac{1}{M}\sum_{m\leq M}\log f_m(\rvx)_{\rvy}$, the $\tau\to\infty$ limit of the soft assignment in \cref{eq:awta}.
It falls below COMPOTE on every dataset and metric, so the gains are not explained by architectural multiplicity or naive ensembling alone.\\
\emph{Not a mixture of experts:} MCT-Router follows the routing recipe of multi-interest recommendation (\cref{sec:introduction}): it keeps the shared encoder and replaces the training of COMPOTE by a learned gate that selects one expert per query.
It trails COMPOTE everywhere (N@5 of $2.28$ against $2.45$ on Instrument, $1.64$ against $1.92$ on Scientific, $3.78$ against $4.01$ on Game).
The next section traces this gap to head collapse: without the annealed winner-takes-all objective, routing fails to induce distinct hypotheses (\cref{sec:collapse-oracle}).

\begin{table*}[t]
    \caption{Ablation over the prediction head and the training objective (\%).
    Variants share the same encoder backbone and parameter budget, and rank items with max-similarity aggregation (\cref{eq:max-similarity-aggregation}):
    MCT-SASRec is a single-hypothesis iso-parameter head, MCT-EnsCE uses an ensemble cross-entropy objective, and MCT-Router routes queries to a single expert.
    \textbf{Bold}: best.}
    \label{tab:results-ablation}
    \centering
    \footnotesize
    \setlength{\tabcolsep}{4pt}
    \resizebox{\textwidth}{!}{%
    \begin{tabular}{lcccccccccccc}
        \toprule
        & \multicolumn{4}{c}{Instrument} & \multicolumn{4}{c}{Scientific} & \multicolumn{4}{c}{Game} \\
        \cmidrule(lr){2-5} \cmidrule(lr){6-9} \cmidrule(lr){10-13}
        Method & H@5 & H@10 & N@5 & N@10 & H@5 & H@10 & N@5 & N@10 & H@5 & H@10 & N@5 & N@10 \\
        \midrule
        MCT-SASRec & 3.63 & 5.70 & 2.34 & 2.97 & 2.55 & 3.87 & 1.71 & 2.13 & 5.73 & 8.84 & 3.78 & 4.78 \\
        MCT-EnsCE & 3.54 & 5.53 & 2.28 & 2.96 & 2.42 & 3.80 & 1.60 & 2.03 & 5.74 & 8.81 & 3.78 & 4.77 \\
        MCT-Router & 3.57 & 5.48 & 2.28 & 2.89 & 2.40 & 3.60 & 1.64 & 2.00 & 5.74 & 8.81 & 3.78 & 4.77 \\
        \midrule
        COMPOTE (ours) & \textbf{3.75} & \textbf{5.83} & \textbf{2.45} & \textbf{3.10} & \textbf{2.84} & \textbf{4.20} & \textbf{1.92} & \textbf{2.35} & \textbf{6.06} & \textbf{9.26} & \textbf{4.01} & \textbf{5.04} \\
        \bottomrule
    \end{tabular}%
    }
\end{table*}

\section{Hypothesis Specialization}
\label{sec:hypothesis-specialization}

We first show that, unlike the matched-capacity variants of \cref{sec:ablations}, COMPOTE learns hypotheses that do not collapse and carry complementary rankings.
We then explain this phenomenon by characterizing what each hypothesis $\rvp^\star_m$ learns at the winner-takes-all limit $\tau\to0$.

\subsection{Collapse and Oracle Evidence}
\label{sec:collapse-oracle}

We quantify collapse through the mean pairwise cosine similarity of the head outputs.
For a query $\rvx$, let $\bar{\rvh}_m(\rvx)=\rvh_m(\rvx)/\|\rvh_m(\rvx)\|_2$ and define $S(\rvx) = \frac{1}{M(M-1)}\sum_{j,k=1}^{M} \bar{\rvh}_j(\rvx)^\top \bar{\rvh}_k(\rvx)\,$.
We report its test-set average; $S=1$ means that all heads have the same direction.
MCT-EnsCE and MCT-Router reach $S=1.00$ on all three datasets, whereas COMPOTE obtains $0.64$, $0.61$ and $0.58$ on \textsc{Instrument}, \textsc{Scientific} and \textsc{Game}, respectively (\cref{tab:head-collapse}).
Thus, neither architectural multiplicity nor learned routing is sufficient to prevent collapse.

\paragraph{Oracle head selection.}
To test whether COMPOTE's distinct heads contain complementary rankings, we compare the max-similarity rule of \cref{eq:max-similarity-aggregation} with an oracle that observes the target and selects the head assigning it the highest similarity,
\begin{equation}
    \label{eq:oracle-head}
    m^{\rm orc}(\rvx,\rvy)=\argmax_{m\leq M}s_{m,\rvy}(\rvx)\,,
\end{equation}
then ranks every item using $s_{m^{\rm orc}(\rvx,\rvy),i}(\rvx)$.
The oracle is not deployable, but it measures whether one hypothesis ranks the target substantially better than the aggregated list.
For MCT-EnsCE and MCT-Router, aggregated and oracle N@5 coincide exactly on every dataset: selecting among collapsed heads cannot help (\cref{tab:oracle-gap}).
For COMPOTE, oracle selection raises N@5 by $53\%$ on \textsc{Instrument} ($2.45\to3.75$), $83\%$ on \textsc{Scientific} ($1.92\to3.52$) and $94\%$ on \textsc{Game} ($4.01\to7.78$).
The ranking gains therefore coincide with a qualitative change in the learned predictor: COMPOTE produces distinct heads whose rankings contain complementary signal.

These findings raise a structural question: why does COMPOTE's training create complementary hypotheses, whereas ensemble and routed objectives collapse?

\subsection{What Each Hypothesis Learns}
\label{sec:what-hypotheses-learn}

For a fixed query, \cref{thm:closed-form} describes the optimal hypotheses through three properties:
\begin{enumerate}
    \item[(P1)] \emph{Partition.} The catalog splits into disjoint cells $\cY_1,\ldots,\cY_M$, each hypothesis supported on its own cell.
    \item[(P2)] \emph{Faithfulness within cells.} On $\cY_m$, $\rvp^\star_m$ is $\rvq^\star$ renormalized by $Q_m$: each hypothesis ranks the items of its cell exactly as the Bayes predictor does, and therefore still depends on the query through $\rvq^\star$.
    \item[(P3)] \emph{Balance.} The partition maximizes the entropy of $(Q_m)_{m=1}^M$, so the cells split the mass of $\rvq^\star$ as evenly as the catalogue allows.
\end{enumerate}

\begin{restatable}[Extension of Th. 1 -- \cite{rupprecht2017mhp}]{theorem}{thmClosedForm}
    \label{thm:closed-form}
    Let $\rvq^\star \in \Delta^{N\mhyphen1}$, and denote the optimal prediction $\rvp^\star \in \argmin_{\rvp\in\left(\Delta^{N\mhyphen1}\right)^M} \lE_{\rvy\sim\rvq^\star}[\ell^\textsc{wta}(\rvy, \rvp)]$.
    Then for every $m\leq M$ and $y\in\cY$,
    \begin{align}
        \label{eq:closed-form}
        \rvp^\star_{m,y}
        = \frac{\rvq^\star_y}{Q_m}\indicator{y\in \cY_m}\,,
        \quad\text{with}\quad
        Q_m = \sum_{y\in\cY_m} \rvq^\star_{y}\,.
    \end{align}
    where the tessalation is $\cY_m = \left\{y \in \cY : ~ \forall m' \neq m,~ -\log(\rvp^\star_{m,y}) < -\log(\rvp^\star_{m',y})\right\}$. Moreover the tessalation $(\cY_m)_{m=1}^M$ is entropy-maximizing.
\end{restatable}

Proof can be found in \cref{ssec:thm:proof:closed-form}.

COMPOTE is trained at $\tau>0$ with softmax heads, so the exact zeros of \cref{eq:closed-form} cannot be reached.
We test each property through its observable signature.
Mirroring the tessellation of \cref{thm:closed-form}, the empirical cell of head $m$ on query $\rvx$ is $\cY_m(\rvx)=\{y\in\cY : m=\argmax_{m'} f_{m',y}(\rvx)\}$, and the \emph{winning head} of a test pair $(\rvx,\rvy)$ is the head whose cell contains $\rvy$.
All diagnostics are computed on \textsc{Game}, with $M=3$ unless stated otherwise, and \cref{tab:latent-tiling-theory} reports them for $M\in\{3,4,5\}$.
The protocol is detailed in \cref{app:latent-tiling}, which also reports the same partition on \textsc{Scientific}.

\begin{table}[h]
    \caption{Signatures of \cref{thm:closed-form} on \textsc{Game} as the number of hypotheses $M$ varies.
    All columns are computed on the full test set of $94{,}762$ query--target pairs.
    Win frequencies and sector shares are reported as the range over heads.}
    \label{tab:latent-tiling-theory}
    \centering
    \small
    \setlength{\tabcolsep}{4pt}
    \begin{tabular}{ccccccc}
        \toprule
        $M$ & In-cell mass & Win frequency (\%) & Sector share (\%) & $\hat\cL^\textsc{wta}$ & $\hat\cL^\textsc{wta}+\log M$ & Sector agreement \\
        \midrule
        3 & $0.929$ & $[32.0,\,34.4]$ & $[27.9,\,36.3]$ & $7.32$ & $8.41$ & $98.5\%$ \\
        4 & $0.915$ & $[24.5,\,25.6]$ & $[23.4,\,27.1]$ & $7.02$ & $8.41$ & $98.7\%$ \\
        5 & $0.937$ & $[18.9,\,21.1]$ & $[15.5,\,25.3]$ & $6.71$ & $8.32$ & $99.4\%$ \\
        \bottomrule
    \end{tabular}
\end{table}

\paragraph{(P1) The hypotheses partition the catalogue.}
The non-collapsed head directions of \cref{tab:head-collapse} are a first signature of (P1).
The partition holds for the full distributions and not only for the top of the rankings: each head keeps between $91.5\%$ and $93.7\%$ of its softmax mass inside its own cell, for every $M$.
The retrieved sets are therefore nearly disjoint.
The top-$10$ lists of two heads on the same query have a mean Jaccard overlap of $0.008$, and $91.8\%$ of the catalogue items retrieved at test time are reached by a single head only.
Both sharpen with $M$, reaching $0.001$ and $98.6\%$ at $M=5$.
\Cref{fig:head-partition} shows the resulting partition: projected onto the plane of the head directions, the catalogue splits into three sectors, one per hypothesis.

\begin{figure}[!htbp]
    \centering
    \begin{subfigure}[t]{0.32\textwidth}
        \centering
        \includegraphics[width=\linewidth]{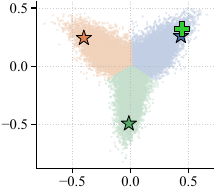}
    \end{subfigure}\hfill
    \begin{subfigure}[t]{0.32\textwidth}
        \centering
        \includegraphics[width=\linewidth]{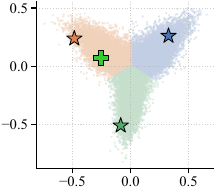}
    \end{subfigure}\hfill
    \begin{subfigure}[t]{0.32\textwidth}
        \centering
        \includegraphics[width=\linewidth]{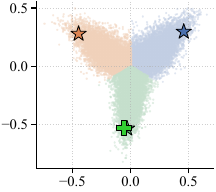}
    \end{subfigure}
    \caption{\textbf{Catalogue tiling on \textsc{Game} ($M=3$).}
    Normalized items (dots) and head outputs (stars) are projected onto the plane spanned by the static head directions $\bar{\rvh}_m$, the average normalized output of each head over the test set.
    Items are coloured by the direction they are most aligned with, and the green cross marks the target.
    Each panel is a test query won by a different head, and in each the target lies in the sector of its winning head.}
    \label{fig:head-partition}
\end{figure}

\paragraph{(P2) Each hypothesis remains query-dependent.}
By (P2), a hypothesis is not a fixed list of items: its ranking follows $\rvq^\star(\rvx)$.
Empirically, the top-$10$ of a given head overlaps by only $0.017$ to $0.026$ Jaccard between two different histories, and in \cref{fig:head-partition} the head closest to the target changes from query to query.
(P1) and (P2) together explain the oracle gap in \cref{tab:oracle-gap}.
At the optimum, the target $\rvy$ belongs to a single cell $\cY_m$, where it competes only with the other items of $\cY_m$, ranked by $\rvq^\star$.
Selecting that head therefore ranks $\rvy$ no worse than the Bayes predictor over the full catalogue.
The cosine oracle of \cref{eq:oracle-head} makes exactly this choice, as it picks the winning head on $98.9\%$ to $99.7\%$ of test pairs.
This is the behavior of COMPOTE, where the oracle-selected head ranks the target substantially higher than the pooled list, while collapsed models exhibit no oracle gain.

\paragraph{(P3) The cells split the mass evenly.}
The win frequency of a head (the fraction of test targets it wins), estimates $\lE_\rvx[Q_m(\rvx)]$ without access to $\rvq^\star$.
Win frequencies stay within $1.6$ points of $1/M$ for every $M$, and their normalized entropy exceeds $0.999$.
Balance holds in mass and not in item counts: the heads own between $15.5\%$ and $25.3\%$ of the catalogue at $M=5$, yet each wins close to $20\%$ of the targets.
This is the balance (P3) predicts, since $Q_m$ weights items by $\rvq^\star$, so a head with fewer items holds items that are more often the next one.
Balance also matters for aggregation: if $Q_m=1/M$ for all $m$, \cref{eq:closed-form} gives $\omega(\rvp^\star)=M\rvq^\star$, so max-pooling recovers the Bayes ranking exactly.
Balance also fixes the value of the optimal loss, and \cref{ssec:log-m-law} checks that the test loss follows the predicted $\log M$ law.

\section{Discussion and Conclusion}
\label{sec:discussion-conclusion}

We introduced COMPOTE, which uses annealed multiple-hypothesis learning to produce complementary rankings for sequential recommendation.
Our theoretical and empirical results show that it prevents head collapse, preserves ranking quality, and improves performance across three benchmarks.
We leave combining COMPOTE with semantic identifiers for future work, which could improve recommendations for rare items.
It also remains to explain why its theoretically query-dependent cells form an almost static catalogue tiling and how this structure relates to item-embedding geometry (\cref{app:static-tiling}).

\subsection*{AI use statement}
We used generative AI tools to polish the manuscript, primarily by correcting grammar and improving clarity, 
and to assist with portions of the experimental code. Parts of the theoretical proofs (the generalization bounds) were also drafted with the help 
of generative AI, but all proofs were verified and corrected by the authors.

\subsection*{Acknowledgments}
This project was provided with computing HPC and storage resources by GENCI at IDRIS thanks to the grant 2026-A0191016862 on the supercomputer Jean Zay's V100 partition.

\bibliographystyle{plainnat}
\bibliography{main}

\appendix

\section{Proofs - \cref{sec:problem-statement-learning-problem}}
\label{sec:problem-statement-learning-problem-proofs}

\subsection{Preliminary results}

\begin{restatable}[$\ell_\tau$ as a convex conjugate]{proposition}{propLogSumExpConjugate}
    \label{prop:log-sum-exp-conjugate}
    For any $(\rvy,\rvp) \in \cY\times \left(\Delta^{N\mhyphen1}\right)^M$ and for any $\tau \geq 0$,
    \begin{align}
    \label{eq:log-sum-exp-conjugate}
    \ell_\tau(\rvy, \rvp) = \min_{\gamma\in\Delta^{M\mhyphen1}} \sum_{m\leq M} -\gamma \log(\rvp_{m,\rvy}) - \tau H(\gamma)\,.
\end{align}
with optimal $\gamma$ given by
\begin{align}
    \label{eq:gamma-optimal}
    \gamma(\rvy, \rvp) = \softmax\left(\frac{1}{\tau}\log(\rvp_{\cdot,\rvy})\right)\,.
\end{align}
\end{restatable}

\begin{proof}
    The log-sum-exp is the convex conjugate of the negative entropy.
\end{proof}

\subsection{Proof of \cref{prop:tau-to-wta}}
\label{ssec:prop:proof:tau-to-wta}

\propTauToWta*
\begin{proof}[Proof of \cref{prop:tau-to-wta}]
    The proof is a direct consequence of \cref{prop:log-sum-exp-conjugate} as $\lim_{\tau \to 0^+}\softmax\left(\frac{1}{\tau}\log(\rvp_{\cdot, \rvy})\right) = \max_{m\leq M}\frac{1}{\tau}\log(\rvp_{m, \rvy})$.
\end{proof}

\subsection{Proof of \cref{prop:loss-tau-to-awta}}
\label{ssec:prop:proof:loss-tau-to-awta}

\propLossTauToAWta*
\begin{proof}[Proof of \cref{prop:loss-tau-to-awta}]
    This is a direct consequence of \cref{prop:log-sum-exp-conjugate}.
\end{proof}

\subsection{Proof of \cref{cor:awta-gradient}}
\label{ssec:prop:proof:awta-gradient}

\corAwtaGradient*
\begin{proof}[Proof of \cref{cor:awta-gradient}]
    \begin{align}
        \nabla_\theta \psi_\tau(\rvy,\rvp^\theta,\alpha)|_{\alpha = \gamma_m(\rvy, \rvp^\theta)}
        = -\sum_{m\leq M} \alpha_m \frac{\nabla_\theta \rvp^\theta_{m,\rvy}}{\rvp^\theta_{m,\rvy}}
        = -\sum_{m\leq M} \frac{(\rvp^\theta_{m,\rvy})^{1/\tau - 1}}{\sum_{m'\leq M} (\rvp^\theta_{m',\rvy})^{1/\tau}}
        \nabla_\theta \rvp^\theta_{m,\rvy}\,.
    \end{align}
    On the other hand, differentiating $\ell_\tau(\rvy,\rvp^\theta) = -\tau \log\big(\sum_{m\leq M}
    (\rvp^\theta_{m,\rvy})^{1/\tau}\big)$ directly gives
    \begin{align}
        \nabla_\theta \ell_\tau(\rvy,\rvp^\theta)
        = -\tau \frac{\sum_{m\leq M} \frac{1}{\tau} (\rvp^\theta_{m,\rvy})^{1/\tau - 1}\nabla_\theta \rvp^\theta_{m,\rvy}}
        {\sum_{m'\leq M} (\rvp^\theta_{m',\rvy})^{1/\tau}}
        = -\sum_{m\leq M} \frac{(\rvp^\theta_{m,\rvy})^{1/\tau - 1}}{\sum_{m'\leq M} (\rvp^\theta_{m',\rvy})^{1/\tau}}
        \nabla_\theta \rvp^\theta_{m,\rvy}\,.
    \end{align}
    This concludes the proof.
\end{proof}

\section{Proofs of \cref{sec:compote-to-ranking}}
\label{sec:proofs:compote-to-ranking}

\subsection{Proof of \cref{prop:loss-via-pooling}}
\label{ssec:proof:prop:loss-via-pooling}

\propLossViaPooling*
\begin{proof}[Proof of \cref{prop:loss-via-pooling}]
    Fix $(\rvy,\rvp)$. For $\tau > 0$, by definition of $\omega^\tau$ in \cref{eq:loss-via-pooling},
    \begin{align}
        -\log\left(\omega^\tau_\rvy\left(\rvp\right)\right)
        = -\log\left(\left(\sum_{m\leq M} \rvp_{m,\rvy}^\frac{1}{\tau}\right)^\tau\right)
        = -\tau\log\left(\sum_{m\leq M} \rvp_{m,\rvy}^\frac{1}{\tau}\right)
        = \ell_\tau(\rvy,\rvp)\,,
    \end{align}
    the last equality being the definition of $\ell_\tau$ in \cref{eq:loss-free-energy}.

    For $\tau = 0$, understood as the limit $\tau\to 0^+$ on both sides, $\omega^0 = \omega$ gives $-\log\big(\omega^0_\rvy(\rvp)\big) = -\log\max_{m\leq
    M} \rvp_{m,\rvy} = \ell^\textsc{wta}(\rvy,\rvp)$, which equals $\lim_{\tau\to0^+}
    \ell_\tau(\rvy,\rvp)$ by \cref{prop:tau-to-wta}.
\end{proof}

\subsection{Proof of \cref{prop:emax-ranking}}
\label{ssec:proof:prop:wta-emax-ranking}

\propEmaxRanking*

\begin{proof}[Proof of \cref{prop:emax-ranking}]
    Denote the inner risk by $L(\rvp) = \lE_{\rvy\sim\rvq^\star}[\ell^\textsc{wta}(\rvy, \rvp)]$. By
    \cref{prop:loss-via-pooling},
    \begin{align}
        \label{eq:prop:emax-ranking:inner-risk}
        L(\rvp) = -\sum_{y\in\cY}\rvq^\star_y \log \omega_y(\rvp)\,,
    \end{align}
    so that, $L(\rvp^\star)$ being finite, $\omega_y(\rvp^\star) > 0$ for every $y$ with
    $\rvq^\star_y > 0$. Fix $a, b \in \cY$ with $\rvq^\star_a > \rvq^\star_b$; in particular
    $\rvq^\star_a > 0$ and $\omega_a(\rvp^\star) > 0$.

    \emph{Weak inequality.} Assume $\omega_b(\rvp^\star) > 0$, the conclusion being trivial
    otherwise, and let $\Pi_{ab}$ be the operator associated with the transposition $\pi = (a\,b)$.
    Then
    \begin{align}
        L(\Pi_{ab}\rvp^\star) - L(\rvp^\star)
        &= -\rvq^\star_a\log\omega_b(\rvp^\star) - \rvq^\star_b\log\omega_a(\rvp^\star)
           + \rvq^\star_a\log\omega_a(\rvp^\star) + \rvq^\star_b\log\omega_b(\rvp^\star) \\
        &= (\rvq^\star_a - \rvq^\star_b)\log\frac{\omega_a(\rvp^\star)}{\omega_b(\rvp^\star)}\,,
    \end{align}
    the terms at $y \notin \{a, b\}$ cancelling and the remaining logarithms being finite. Since
    $L(\Pi_{ab}\rvp^\star) \geq L(\rvp^\star)$ and $\rvq^\star_a > \rvq^\star_b$, this gives
    $\omega_a(\rvp^\star) \geq \omega_b(\rvp^\star)$.

    \emph{Strictness up to saturation.} Assume now $\omega_a(\rvp^\star) = \omega_b(\rvp^\star)$; we
    show that this common value is $1$. Let $(\cY_m)_{m\leq M}$ be the tessellation of
    \cref{thm:closed-form}, $m(y)$ the index of the cell containing $y$ and $Q_m$ the mass of
    $\cY_m$, so that $\omega_y(\rvp^\star) = \rvp^\star_{m(y),y} = \rvq^\star_y / Q_{m(y)}$ whenever
    $Q_{m(y)} > 0$, and, by \eqref{eq:prop:emax-ranking:inner-risk},
    \begin{align}
        \label{eq:prop:emax-ranking:risk-of-cells}
        L(\rvp^\star)
        = -\sum_{y\in\cY} \rvq^\star_y\log\frac{\rvq^\star_y}{Q_{m(y)}}
        = H(\rvq^\star) + \sum_{m\leq M} \phi(Q_m)\,,
        \qquad \phi(t) = t\log t,\quad \phi(0) = 0\,.
    \end{align}

    Suppose first $\rvq^\star_b > 0$, so that $Q_A, Q_B > 0$ where $A = m(a)$ and $B = m(b)$. The
    items $a$ and $b$ lie in distinct cells: were they in the same cell, the ratio
    $\omega_a(\rvp^\star)/\omega_b(\rvp^\star) = \rvq^\star_a/\rvq^\star_b > 1$ would contradict the
    assumed equality. Exchanging $a$ and $b$ between the two cells produces a partition of $\cY$
    with masses $Q_A - u$ and $Q_B + u$, where $u = \rvq^\star_a - \rvq^\star_b > 0$, the other cells
    being unchanged; the predictor it defines through \eqref{eq:closed-form} is admissible, so
    \eqref{eq:prop:emax-ranking:risk-of-cells} and the optimality of $\rvp^\star$ give
    \begin{align}
        \label{eq:prop:emax-ranking:exchange}
        g(u) \geq g(0)\,,
        \qquad\text{where}\qquad
        g(t) = \phi(Q_A - t) + \phi(Q_B + t)\,.
    \end{align}
    Writing $\Delta = Q_A - Q_B$, we have $g'(t) = \log\frac{Q_B + t}{Q_A - t}$, so $g$ decreases on
    $(-\infty, \Delta/2]$ and increases on $[\Delta/2, Q_A)$; being symmetric about $\Delta/2$, it
    satisfies $g(u) \geq g(0)$ with $u > 0$ only if $u \geq \Delta$. Consequently
    \begin{align}
        \rvq^\star_a Q_B - \rvq^\star_b Q_A
        \;\geq\; (\rvq^\star_b + \Delta) Q_B - \rvq^\star_b (Q_B + \Delta)
        \;=\; \Delta\,(Q_B - \rvq^\star_b)
        \;\geq\; 0\,,
    \end{align}
    using $\rvq^\star_a \geq \rvq^\star_b + \Delta$ and $Q_B \geq \rvq^\star_b$ if $\Delta > 0$, and
    $\rvq^\star_a > \rvq^\star_b$ with $Q_A \leq Q_B$ if $\Delta \leq 0$. The assumed equality reads
    $\omega_a(\rvp^\star) - \omega_b(\rvp^\star) = (\rvq^\star_a Q_B - \rvq^\star_b Q_A)/(Q_A Q_B) =
    0$ and therefore forces $\Delta > 0$, $\rvq^\star_a = \rvq^\star_b + \Delta$ and $Q_B =
    \rvq^\star_b$, whence $\omega_b(\rvp^\star) = \rvq^\star_b/Q_B = 1$.

    Suppose finally $\rvq^\star_b = 0$. If $Q_B > 0$ then $\omega_b(\rvp^\star) = \rvq^\star_b/Q_B =
    0 < \omega_a(\rvp^\star)$, excluding the equality. Otherwise every $y \in \cY_B$ has
    $\rvq^\star_y = 0$; let $\rvp'$ be $\rvp^\star$ with head $B$ replaced by $e_a$. For $y \notin
    \cY_B \cup \{a\}$ the maximum defining $\omega_y$ is attained by a head other than $B$, so
    $\omega_y(\rvp') = \omega_y(\rvp^\star)$; the items of $\cY_B$ carry no mass; and
    $\omega_a(\rvp') = 1 \geq \omega_a(\rvp^\star)$. Hence $L(\rvp') \leq L(\rvp^\star)$ with
    equality only if $\omega_a(\rvp^\star) = 1$, and the optimality of $\rvp^\star$ forces
    $\omega_a(\rvp^\star) = 1$.
\end{proof}

\section{Per-Head Normalizer}
\label{app:scaling-normalization}

\paragraph{Protocol.}
For the $M=3$ \textsc{Game} checkpoint of \cref{tab:aggregation}, we compute
$\log Z_m(\rvx)=\log\sum_{i\leq N}\exp(s_{m,i}(\rvx))$ over all $N=25{,}612$ items for each of the $94{,}762$ test queries, with $s_{m,i}=\cos(h_m,\rve_i)/\beta$ and $\beta=0.07$.

\paragraph{The normalizer is head-independent.}
The three distributions nearly overlap (\cref{fig:z-m-repartition}), with means $13.67$, $13.63$, and $13.79$.
At fixed $\rvx$, the cross-head standard deviation of $\log Z_m(\rvx)$ is $0.13$ on average; its range has median $0.28$, $95$th percentile $0.62$, and maximum $1.10$.

\paragraph{Effect on the ranking.}
Because $\log\max_m\softmax_i(\rvs_{m,\cdot})=\max_m(\rvs_{m,i}-\log Z_m)$, a ranking change requires max-cosine scores within $\beta$ times the cross-head range: $0.020$ for the median query, $0.043$ at the $95$th percentile, and $0.077$ at most.
Accordingly, \cref{tab:aggregation} shows no measurable loss from dropping the normalizer.

\begin{figure}[!htbp]
    \centering
    \includegraphics[width=0.6\linewidth]{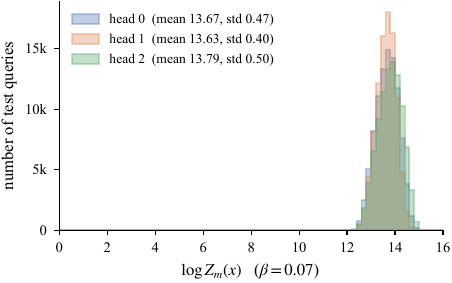}
    \caption{\textbf{Distribution of the per-head log-normalizer $\log Z_m(\rvx)$.}
    Histogram of $\log Z_m(\rvx)=\log\sum_{i\leq N}\exp\big(s_{m,i}(\rvx)\big)$ over the $94{,}762$ test queries of \textsc{Game}, one colour per head, for the COMPOTE checkpoint with $M=3$ of \cref{tab:aggregation} ($\beta=0.07$, $N=25{,}612$).
    The three heads have almost superimposed distributions, concentrated between $12.2$ and $15.0$ on a scale where $\log Z_m$ could a priori take any value in $[\log N - 1/\beta,\ \log N + 1/\beta]=[-4.1, 24.4]$.
    For a fixed query, the standard deviation of $\log Z_m(\rvx)$ across the three heads is $0.13$ on average, against a mean value of $13.7$.}
    \label{fig:z-m-repartition}
\end{figure}

\section{Generalization Bound}
\label{app:generalization-bound}

\subsection{Proof of \cref{thm:wta-generalization}}
\label{app:proof-wta-generalization}
We first provide here a generalization bound for the
log-loss in \cref{thm:wta-generalization}, which is a direct consequence of
 the Rademacher complexity of the hypothesis class $\mathcal{F}^M$.
 We then provide a refined version of the generalization bound
 in \cref{cor:wta-generalization-refined}, which has a milder
  dependence on the number of heads $M$.

\begin{restatable}[log-loss generalization bound]{theorem}{thmWtaGeneralization}
    \label{thm:wta-generalization}
    Let $\mathcal{F}$ be a hypothesis class of probability-generating heads
    $f_{m}: \mathcal{X} \to \Delta^{N\mhyphen1}$. Assume there exists a constant $\epsilon > 0$
    such that $f_{m,i}(\mathbf{x}) \ge \epsilon$ for all $m \in [M]$, $\mathbf{x} \in \mathcal{X}$, and $i \in [N]$.
    Then, for any $\delta \in (0, 1)$, with probability at least $1 - \delta$ over the sample
    $S \sim P_{\mathbf{x}, \mathbf{y}}^n$, every ensemble $f = (f_{1}, \dots, f_{M}) \in \mathcal{F}^M$ satisfies:
    \begin{equation}
        \label{eq:wta-generalization-bound}
        \mathcal{L}^\textsc{wta}(f, P_{\mathbf{x}, \mathbf{y}}) \le \widehat{\mathcal{L}}^\textsc{wta}_S(f) + \frac{2}{\epsilon} \sum_{m=1}^M \mathcal{R}_n(\mathcal{F}_{m,\mathbf{y}}) +  \log\left(\frac{1}{\epsilon}\right) \sqrt{\frac{\ln(2/\delta)}{2n}},
    \end{equation}
    where $\mathcal{R}_n(\mathcal{F}_{m,\mathbf{y}})$ is the expected Rademacher complexity of the
    set of functions  $\mathcal{F}_{m,\mathbf{y}} = \{ \mathbf{x} \mapsto f_{m,\mathbf{y}}(\mathbf{x}) : f_{m} \in \mathcal{F} \}$.
\end{restatable}

\begin{proof}[Proof of \cref{thm:wta-generalization}]
\textbf{Step 1: Boundedness and Lipschitz continuity of the loss function.} \\
Consider the scalar loss function $ -\log(t)$ restricted to $t \in [\epsilon, 1]$.
\begin{itemize}
    \item \textbf{Upper Bound ($B$):} Since $t \ge \epsilon$, the loss is non-negative and bounded above by:
    \[
    B = \sup_{t \in [\epsilon, 1]} (-\log t) = -\log(\epsilon) = \log\left(\frac{1}{\epsilon}\right).
    \]
    Thus, for every sample $(\mathbf{x}_i, \mathbf{y}_i)$, $\min_{m \in [M]} (-\log f_{m,\mathbf{y}_i}(\mathbf{x}_i)) \in [0, B]$.
    \item \textbf{Lipschitz Continuity ($L_\epsilon$):} By the Mean Value Theorem, the Lipschitz constant is bounded by the supremum of the derivative:
    \[
    L_\epsilon = \sup_{t \in [\epsilon, 1]} \left| \frac{d}{dt} (-\log t) \right| = \sup_{t \in [\epsilon, 1]} \frac{1}{t} = \frac{1}{\epsilon}.
    \]
    Therefore, for any $t_1, t_2 \ge \epsilon$, we have $|(-\log t_1) - (-\log t_2)| \le \frac{1}{\epsilon} |t_1 - t_2|$.
\end{itemize}

\noindent \textbf{Step 2: Concentration via McDiarmid's Inequality.} \\
We want to show that the empirical loss on $n$ finite samples
$\widehat{\mathcal{L}}^\textsc{wta}_S(f)$ is close to the true expected loss $\mathcal{L}^\textsc{wta}(f)$ uniformly across all possible models in our hypothesis class $\mathcal{F}^M$.
Define the worst-case generalisation gap function:
\[
\Phi(S) = \sup_{f \in \mathcal{F}^M} \left( \mathcal{L}^\textsc{wta}(f, P_{\mathbf{x}, \mathbf{y}}) - \widehat{\mathcal{L}}^\textsc{wta}_S(f) \right).
\]
Suppose we change a single data point in $S = \{(\mathbf{x}_1, \mathbf{y}_1), \dots, (\mathbf{x}_n, \mathbf{y}_n)\}$
to create a neighboring dataset
$S' = \{(\mathbf{x}_1, \mathbf{y}_1), \dots, (\mathbf{x}_i', \mathbf{y}_i'), \dots, (\mathbf{x}_n, \mathbf{y}_n)\}$.
The change in $\Phi(S)$ can be bounded as:
\[
\left\vert \Phi(S) - \Phi(S') \right\vert \le \sup_{f \in \mathcal{F}^M} \frac{1}{n} \left\vert \min_{m} (-\log f_m(\mathbf{x}_i)_{\mathbf{y}_i}) - \min_{m} (-\log f_m(\mathbf{x}_i')_{\mathbf{y}_i'}) \right\vert \le \frac{B}{n} = \frac{\log(1/\epsilon)}{n}.
\]
Because changing one sample alters $\Phi(S)$ by at most $c_i = \frac{\log(1/\epsilon)}{n}$,
we can apply McDiarmid's inequality to bound the deviation of $\Phi(S)$ from its expectation
\citep{bousquet2002stability}:
\[
\mathbb{P} \left( \Phi(S) - \mathbb{E}_S[\Phi(S)] \ge t \right) \le \exp\left( \frac{-2 t^2}{\sum_{i=1}^n c_i^2} \right).
\]
Evaluating $\sum_{i=1}^n c_i^2 = n \cdot \left(\frac{\log(1/\epsilon)}{n}\right)^2 = \frac{\log^2(1/\epsilon)}{n}$.
Set the probability tail $\exp\left( \frac{-2 t^2 n}{\log^2(1/\epsilon)} \right) = \frac{\delta}{2}$.
Then, solving for $t$ gives
\[
t = \log\left(\frac{1}{\epsilon}\right) \sqrt{\frac{\ln(2/\delta)}{2n}}.
\]
Thus, with probability at least $1 - \frac{\delta}{2}$:
\[
\sup_{f \in \mathcal{F}^M} \left( \mathcal{L}^\textsc{wta}(f, P_{\mathbf{x}, \mathbf{y}}) - \widehat{\mathcal{L}}^\textsc{wta}_S(f) \right) \le \mathbb{E}_S \left[ \sup_{f \in \mathcal{F}^M} \left( \mathcal{L}^\textsc{wta} - \widehat{\mathcal{L}}^\textsc{wta}_S \right) \right] + \log\left(\frac{1}{\epsilon}\right) \sqrt{\frac{\ln(2/\delta)}{2n}}.
\]

\noindent \textbf{Step 3: Symmetrization and min-operator decomposition.} \\
Our goal now is to bound the expected supremum
$\mathbb{E}_S   \left[\sup_{f \in \mathcal{F}^M} \left( \mathcal{L}^\textsc{wta} - \widehat{\mathcal{L}}^\textsc{wta} \right) \right]$ in terms of the Rademacher complexity of the hypothesis class $\mathcal{F}^M$.
For doing so, we can redefine $\mathbb{E}_S[\Phi(S)]$ as follows:
\[
\mathbb{E}_S[\Phi(S)] = \mathbb{E}_S \left[ \sup_{f \in \mathcal{F}^M} \left( \mathbb{E}_{S'} [\widehat{\mathcal{L}}^\textsc{wta}_{S'}(f)] - \widehat{\mathcal{L}}^\textsc{wta}_S(f) \right) \right],
\]
where $S'$ is an independent ghost sample drawn from the same distribution.
Then by gathering the terms, and using Jensen's inequality, we can pull the expectation over $S'$ outside of the supremum:
\[
\mathbb{E}_S[\Phi(S)] \le \mathbb{E}_{S, S'} \left[ \sup_{f \in \mathcal{F}^M} \frac{1}{n} \sum_{i=1}^n \left( \ell(f; \mathbf{x}_i', \mathbf{y}_i') - \ell(f; \mathbf{x}_i, \mathbf{y}_i) \right) \right].
\]
Now, since $(\mathbf{x}_i, \mathbf{y}_i)$ and $(\mathbf{x}_i', \mathbf{y}_i')$ are identically distributed,
introducing independent random sign variables $\sigma_i \in \{-1, +1\}$ with $\mathbb{P}(\sigma_i = 1) = 1/2$
 does not change the distribution of their difference:
 \[
 \mathbb{E}_S[\Phi(S)] \le \mathbb{E}_{S, S'} \left[ \sup_{f} \frac{1}{n} \sum_{i=1}^n \sigma_i \left( \ell(f; \mathbf{x}_i', \mathbf{y}_i') - \ell(f; \mathbf{x}_i, \mathbf{y}_i) \right) \right],
 \]
 and using the triangle inequality leads to the following decomposition:
\begin{align*}
\mathbb{E}_S[\Phi(S)]
&\leq \mathbb{E}_{S', \boldsymbol{\sigma}}
 \left[\sup_f \frac{1}{n}\sum_{i=1}^n
 \sigma_i\ell(f;\mathbf{x}_i',\mathbf{y}_i')\right] \\
&\quad + \mathbb{E}_{S, \boldsymbol{\sigma}}
 \left[\sup_f \frac{1}{n}\sum_{i=1}^n
 (-\sigma_i)\ell(f;\mathbf{x}_i,\mathbf{y}_i)\right] \\
&= 2\,\mathbb{E}_S
 \left[\widehat{\mathcal{R}}_S(\mathcal{G}_{\text{emax}})\right]
 = 2\,\mathcal{R}_n(\mathcal{G}_{\text{emax}}),
\end{align*}
 where $\widehat{\mathcal{R}}_S(\mathcal{G}_{\text{emax}})$ is the empirical Rademacher
 complexity of the composite pointwise minimum class:
 \[
 \mathcal{G}_{\text{emax}} = \left\{ (\mathbf{x}, \mathbf{y})
 \mapsto \min_{m \in [M]} (-\log f_{m,{\mathbf{y}}}(\mathbf{x})) : f_{m,\cdot} \in \mathcal{F} \right\},
 \]
and $\mathcal{R}_n(\mathcal{G}_{\text{emax}}) = \mathbb{E}_S[\widehat{\mathcal{R}}_S(\mathcal{G}_{\text{emax}})]$.

For a fixed sample $S$, using the structural property of Rademacher
 complexity under pointwise minimum operators over $M$ hypothesis classes
 (decomposition of the pointwise minimum as a weighted sum of indicator functions and using results
 from \citet{bartlett2002rademacher}), 
 we bound $\widehat{\mathcal{R}}_S(\mathcal{G}_{\text{emax}})$ as:
\begin{align}
\label{eq:empirical-min-rademacher-bound-sum}
\widehat{\mathcal{R}}_S(\mathcal{G}_{\text{emax}}) \le \sum_{m=1}^M \widehat{\mathcal{R}}_S(-\log \circ \mathcal{F}_{m, y}),
\end{align}
where $\mathcal{F}_{m, y} = \{ \mathbf{x} \mapsto f_{m, y}(\mathbf{x}) : f \in \mathcal{F} \}$.
Since $t \mapsto -\log t$ is $(1/\epsilon)$-Lipschitz on $[\epsilon, 1]$, applying
Talagrand's Contraction Lemma \citep{pmlr-v32-mohri14,ledoux1991probability} yields:
\[
\widehat{\mathcal{R}}_S(-\log \circ \mathcal{F}_{m, y}) \le \frac{1}{\epsilon} \widehat{\mathcal{R}}_S(\mathcal{F}_{m, y}).
\]
Taking expectations $\mathbb{E}_S$ on both sides gives $\mathcal{R}_n(\mathcal{G}_{\text{emax}}) \le \frac{1}{\epsilon} \sum_{m=1}^M \mathcal{R}_n(\mathcal{F}_{m, y})$.

\paragraph{Final combination.}
Substituting into $\mathbb{E}_S[\Phi(S)] \le 2\mathcal{R}_n(\mathcal{G}_{\text{emax}})$ and applying concentration yields the stated bound.
\end{proof}


\subsection{Proof of \cref{cor:wta-generalization-refined}}
\label{app:proof-wta-generalization-refined}
\thmWtaGeneralizationRefined*

\begin{proof}[Proof of \cref{cor:wta-generalization-refined}]

To prove the theorem, we need to show that the expected Rademacher complexity of the minimum operator over $M$ heads
satisfies:
\begin{align}
\mathbb{E}_{S,\boldsymbol{\sigma}} \left[ \sup_{f \in \mathcal{F}^M}
\frac{1}{n} \sum_{i=1}^n \sigma_i \min_{m \in [M]} \left(-\log f_{m,\mathbf{y}}(\mathbf{x}_i)\right) \right] \le \frac{1}{\epsilon} \max_{m \in [M]} \mathcal{R}_n(\mathcal{F}_{m, \mathbf{y}}) + B \sqrt{\frac{2 \ln M}{n}}.
\label{eq:refined-minrademacher-bound}
\end{align}
and plugging this bound into the generalization bound of \cref{thm:wta-generalization} gives the desired result.
Basically, we replace the crude sum over heads in \cref{eq:empirical-min-rademacher-bound-sum} by 
this Massart-type control of the minimum operator, which yields a sharper dependence on~$M$. Hence, 
we only need to prove \cref{eq:refined-minrademacher-bound}. That is done in the following steps.

Let $\mathcal{G}_m = \{ (\mathbf{x}, \rvy) \mapsto -\log f_{m,\mathbf{y}}(\mathbf{x}) : f_m \in \mathcal{F} \}$
 be the loss function class for head $m$, where each function $g_m \in \mathcal{G}_m$ takes
 values in $[0, B]$ with $B = \log(1/\epsilon)$.

Let $g_{m}(\mathbf{z}_i) = -\log f_{m,\mathbf{y}}(\mathbf{x}_i)$, where $\mathbf{z}_i = (\mathbf{x}_i, \rvy_i)$.
 We analyze the empirical Rademacher process for a fixed sample  $S$:
 \[
 \widehat{\mathcal{R}}_S(\mathcal{G}_{\text{emax}})
 = \mathbb{E}_{\boldsymbol{\sigma}} \left[ \sup_{g_1 \in \mathcal{G}_1, \dots, g_M \in \mathcal{G}_M} \frac{1}{n} \sum_{i=1}^n \sigma_i \min_{m \in [M]} g_m(\mathbf{z}_i) \right].
 \]

\noindent \textbf{Pointwise minimum upper bound.} For any given input $\mathbf{z}_i$ and any
 choice of functions $(g_1, \dots, g_M)$, the minimum can be written as:
 \[
 \min_{m \in [M]} g_m(\mathbf{z}_i) = g_{m^*}(\mathbf{z}_i) \quad \text{where } m^* = \arg\min_{m \in [M]} g_m(\mathbf{z}_i).
 \]
 Instead of summing over all $m$, we decouple the choice of the minimizing head index $m^*$ from the individual head function space. For any fixed sequence of functions $(g_1, \dots, g_M)$, we have:
 \[
 \frac{1}{n} \sum_{i=1}^n \sigma_i \min_{m \in [M]} g_m(\mathbf{z}_i) = \frac{1}{n} \sum_{i=1}^n \sigma_i g_{m_i^*}(\mathbf{z}_i).
 \]

\noindent \textbf{$\alpha$-cover discretization of the minima vector.}
To apply Massart's Lemma over the finite $M$ choices of heads,
we construct an empirical $L_2$-net $\mathcal{N}(\alpha, \mathcal{G}_m, S)$ for each function class $\mathcal{G}_m$
at scale $\alpha > 0$. Hence, for every $g_m \in \mathcal{G}_m$,
there exists a net element $\bar{g}_m \in \mathcal{N}(\alpha, \mathcal{G}_m, S)$ such that:
\[
\frac{1}{n} \sum_{i=1}^n \vert g_m(\mathbf{z}_i) - \bar{g}_m(\mathbf{z}_i)\vert^2 \le \alpha^2.
\]
Then, using Cauchy--Schwarz, the discretization error introduced by replacing $g_{m_i^*}$ with
its closest net element $\bar{g}_{m_i^*}$ is bounded by:
\[
\frac{1}{n} \sum_{i=1}^n \sigma_i \left( g_{m_i^*}(\mathbf{z}_i) - \bar{g}_{m_i^*}(\mathbf{z}_i) \right) \le \sqrt{\frac{1}{n} \sum_{i=1}^n \sigma_i^2} \sqrt{\frac{1}{n} \sum_{i=1}^n \vert g_{m_i^*}(\mathbf{z}_i) - \bar{g}_{m_i^*}(\mathbf{z}_i)\vert^2} \le \alpha.
\]

\noindent \textbf{Decomposing into single head and Massart's finite maximum.}
For any tuple of true functions $(g_1, \dots, g_M) \in \mathcal{G}_1 \times \dots \times \mathcal{G}_M$,
we find their corresponding net approximations $(\bar{g}_1, \dots, \bar{g}_M) \in \mathcal{N}^M$.
The empirical sum splits into:
\[
 \frac{1}{n} \sum_{i=1}^n \sigma_i g_{m_i^*}(\mathbf{z}_i) =
\frac{1}{n} \sum_{i=1}^n \sigma_i \bar{g}_{m_i^*}(\mathbf{z}_i) + \frac{1}{n} \sum_{i=1}^n \sigma_i \left( g_{m_i^*}(\mathbf{z}_i) - \bar{g}_{m_i^*}(\mathbf{z}_i) \right).
\]
Taking the supremum over all valid choices of tuples:
\begin{align*}
&\sup_{(g_1,\dots,g_M)\in\prod_m\mathcal{G}_m}
 \frac{1}{n}\sum_{i=1}^n \sigma_i g_{m_i^*}(\mathbf{z}_i) \\
&\quad\leq
 \sup_{(\bar g_1,\dots,\bar g_M)\in\mathcal{N}^M}
 \frac{1}{n}\sum_{i=1}^n \sigma_i\bar g_{m_i^*}(\mathbf{z}_i) \\
&\qquad+
 \sup_{\substack{g_m\in\mathcal{G}_m\\\bar g_m\in\mathcal{N}}}
 \frac{1}{n}\sum_{i=1}^n \sigma_i
 \left(g_{m_i^*}(\mathbf{z}_i)-\bar g_{m_i^*}(\mathbf{z}_i)\right),
\end{align*}
where the final term is at most $\alpha$,
which gives us the following upper bound on the empirical Rademacher process:
\[
R(S) \le \mathbb{E}_{\boldsymbol{\sigma}} \left[ \sup_{(\bar{g}_1, \dots, \bar{g}_M) \in \mathcal{N}^M} \frac{1}{n} \sum_{i=1}^n \sigma_i \bar{g}_{m_i^*}(\mathbf{z}_i) \right] + \alpha.
\]
Since $m_i^* \in [M]$, we decompose $\bar{g}_{m_i^*}(\mathbf{z}_i)$ into an arbitrary reference head (say $m=1$)
plus a residual vector in $\mathbb{R}^n$ representing the candidate selection among $M$ heads:
\[
\bar{g}_{m_i^*}(\mathbf{z}_i) = \bar{g}_1(\mathbf{z}_i) + \left( \bar{g}_{m_i^*}(\mathbf{z}_i) - \bar{g}_1(\mathbf{z}_i) \right).
\]
Then, taking expectation gives the following decomposition of the empirical Rademacher process:
\[
R(S) \le \underbrace{\mathbb{E}_{\boldsymbol{\sigma}} \left[ \sup_{g_1 \in \mathcal{G}_1} \frac{1}{n} \sum_{i=1}^n \sigma_i g_1(\mathbf{z}_i) \right]}_{\text{Empirical Rademacher of single head } \widehat{\mathcal{R}}_S(\mathcal{G}_1)} + \underbrace{\mathbb{E}_{\boldsymbol{\sigma}} \left[ \max_{m \in [M]} \frac{1}{n} \sum_{i=1}^n \sigma_i \mathbf{v}_m(i) \right]}_{\text{Finite maximum over } M \text{ candidate vectors}} +~\alpha,
\]
where $\mathbf{v}_m(i) = \bar{g}_m(\mathbf{z}_i) - \bar{g}_1(\mathbf{z}_i) \in [-B, B]$
as $\bar{g}_m$ and $\bar{g}_1$ are both in the loss function class $\mathcal{G}_m$
and hence bounded by $B = \log(1/\epsilon)$.

For $M$ vectors $\mathbf{v}_1, \dots, \mathbf{v}_M \in \mathbb{R}^n$  with bounded norm $\Vert \mathbf{v}_m\Vert_2 \le \sqrt{n} B$,
 Massart's Finite Class Lemma \citep[Lemma 5.2]{massart2000some} states that:
\[
\mathbb{E}_{\boldsymbol{\sigma}} \left[ \max_{m \in [M]} \frac{1}{n} \sum_{i=1}^n \sigma_i \mathbf{v}_m(i) \right] \le \max_{m \in [M]} \frac{\Vert \mathbf{v}_m\Vert_2}{n} \sqrt{2 \ln M} \le B \sqrt{\frac{2 \ln M}{n}}.
\]
We apply Talagrand's Contraction Lemma on the single-head term and use the fact that the loss function
is $1/\epsilon$-Lipschitz. In addition, by taking the infimum over all valid cover radii $\alpha > 0$ (or letting $\alpha = 1/\sqrt{n}$),
the residual discretization error $\alpha$
is absorbed into the standard covering-number bound of the individual head complexity $\mathcal{R}_n(\mathcal{F}_{1, \mathbf{y}})$.
This gives us the following upper bound on the empirical Rademacher process:
\[
\mathbb{E}_S [\widehat{\mathcal{R}}_S(\mathcal{G}_1)] \le \frac{1}{\epsilon} \mathcal{R}_n(\mathcal{F}_{1,\mathbf{y}}).
\]
Finally, since the choice of head $m=1$ as reference was arbitrary, we take the maximum over all
$m \in [M]$:
\[
\mathbb{E}_{S, \boldsymbol{\sigma}} \left[ \sup_{f \in \mathcal{F}^M} \frac{1}{n} \sum_{i=1}^n \sigma_i \min_{m \in [M]} \left(-\log f_m(\mathbf{x}_i)_{y_i}\right) \right] \le \frac{1}{\epsilon} \max_{m \in [M]} \mathcal{R}_n(\mathcal{F}_{m, \mathbf{y}}) + B \sqrt{\frac{2 \ln M}{n}}.
\]
Substituting this refined bound for $\mathcal{R}_n(\mathcal{G}_{\text{emax}})$ into the symmetrization argument of 
\cref{app:proof-wta-generalization} and combining with McDiarmid's inequality yields the corollary.
\end{proof}

\section{Dataset Description}
\label{sec:dataset-description}

We use three category subsets of the Amazon~2023 review corpus~\citep{hou2024amazon2023} \textsc{Musical Instruments}, \textsc{Industrial and Scientific}, and \textsc{Video Games}, abbreviated Instrument, Scientific, and Game which span May~1996 to September~2023.
Following prior work~\citep{rajput2023tiger,zhou2020s3rec,zheng2025mtgrec}, we apply the benchmark \emph{5-core} filtering~\citep{amazon2023fivecore}, which iteratively discards users and items with fewer than five interactions.
We then group the interactions by user, order them chronologically, and truncate the histories to $50$ items at both training and test time.
The resulting statistics are reported in \cref{tab:dataset-stats}.

\begin{table}[t]
    \caption{Statistics of the preprocessed datasets~\citep{zheng2025mtgrec}. Avg.\,len is the average length of the user interaction sequences.}
    \label{tab:dataset-stats}
    \centering
    \small
    \begin{tabular}{lrrrrr}
        \toprule
        Dataset & \#Users & \#Items & \#Interactions & Sparsity & Avg.\,len \\
        \midrule
        Instrument & 57{,}439 & 24{,}587 & 511{,}836 & 99.964\% & 8.91 \\
        Scientific & 50{,}985 & 25{,}848 & 412{,}947 & 99.969\% & 8.10 \\
        Game & 94{,}762 & 25{,}612 & 814{,}586 & 99.966\% & 8.60 \\
        \bottomrule
    \end{tabular}
\end{table}

\section{Compared Models}
\label{sec:compared-models}

We detail here the baselines summarized in \cref{sec:empirical-evidence}.
Numbers for the sequential and generative recommenders come from \citet{zheng2025mtgrec}, who train them with RecBole~\citep{zhao2021recbole} under a common protocol and tune them by grid search.
The multi-interest models are trained by us.

\paragraph{ID-based sequential recommenders.}
\begin{enumerate}[nosep,leftmargin=*]
    \item \textbf{Caser}~\citep{tang2018caser} applies horizontal and vertical convolutions to the embedded interaction sequence, capturing both positional and skip-behaviour patterns.
    \item \textbf{HGN}~\citep{ma2019hgn} combines feature-level and instance-level gating to modulate item embeddings according to the user representation.
    \item \textbf{GRU4Rec}~\citep{hidasi2016gru4rec} encodes the interaction sequence with gated recurrent units trained for session-based next-item prediction.
    \item \textbf{BERT4Rec}~\citep{sun2019bert4rec} uses a bidirectional transformer encoder trained with a masked-item objective, so that each position is predicted from both past and future context.
    \item \textbf{SASRec}~\citep{kang2018sasrec} uses a unidirectional (causal) self-attention encoder and predicts the next item by inner product with the item embeddings; it is the closest single-hypothesis analogue of our backbone.
    \item \textbf{HSTU}~\citep{zhai2024hstu} introduces hierarchical sequential transducers that consume actions and timestamps jointly, and is designed for scalability while remaining ID-based.
    \item \textbf{S3-Rec}~\citep{zhou2020s3rec} pre-trains the sequence encoder with self-supervised objectives built from item--attribute correlations, then fine-tunes it for next-item prediction.
\end{enumerate}

\paragraph{Generative recommender.}
\textbf{TIGER}~\citep{rajput2023tiger} quantizes item text embeddings into semantic identifiers with an RQ-VAE and casts recommendation as autoregressive generation of the target identifier, retrieving items by beam search over the identifier space.

\paragraph{Multi-interest recommenders.}
These models represent a user by $M$ vectors instead of one, and are the closest existing analogue of our multi-hypothesis parameterization.
We train them with the sampled-softmax objective and the readout rule of their reference implementations, and rank the catalogue by the maximum inner product over the $M$ interest vectors, which is the retrieval rule all of them adopt.
\begin{enumerate}[nosep,leftmargin=*]
    \item \textbf{MIND}~\citep{li2019mind} extracts interests with a dynamic-routing capsule network over the history and trains the interest whose inner product with the target is largest.
    \item \textbf{ComiRec}~\citep{cen2020comirec} comes in two variants: ComiRec-DR keeps the capsule routing of MIND with a bilinear mapping, while ComiRec-SA replaces it by a multi-head self-attentive pooling of the history.
    \item \textbf{SINE}~\citep{tan2021sine} selects a sparse subset of concepts from a large learned concept pool, forms one interest per selected concept, and aggregates them back into a single user vector used for retrieval.
    \item \textbf{Re4}~\citep{zhang2022re4} augments ComiRec-SA with three backward-flow regularizers---contrasting interests against each other, realigning the attention with the retrieval criterion, and reconstructing the assigned items from each interest.
    \item \textbf{REMI}~\citep{xie2023remi} augments ComiRec-SA with interest-aware hard-negative mining in the sampled softmax and a variance-based regularizer on the item-to-interest routing matrix that prevents routing collapse.
\end{enumerate}

\section{Ablation on the Score-Loss Weight}
\label{app:alpha-ablation}

This appendix measures how much of the performance of COMPOTE comes from the score loss of \cref{eq:score-loss}, whose weight in the training objective is~$\alpha$ (\cref{eq:training-loss}).

\paragraph{Protocol.}
We train COMPOTE on \textsc{Scientific} with $M=3$ for $\alpha\in \{0, 0.05, 0.1, 0.2, 0.3, 0.5, 0.7\}$, 
all other hyperparameters being fixed to the values of \cref{sec:empirical-evidence}.
Rankings use max-similarity aggregation (\cref{eq:pooled-prediction}), which does not read $\hat{\gamma}$ at inference.
The score loss can therefore only change the ranking through its gradient on the shared encoder; the oracle assignments it regresses are detached from the computation graph.
We report N@5 at the best checkpoint.

\paragraph{The score loss helps.}
The best weight, $\alpha=0.1$, reaches $1.91$ N@5 against $1.82$ without the score loss (\cref{fig:alpha-ablation}), and four of the six positive weights improve over $\alpha=0$.
We use $\alpha=0.1$ in all other experiments.

\paragraph{It is not the main contributor.}
With $\alpha=0$, COMPOTE already outperforms every single-hypothesis variant of \cref{tab:results-ablation} on \textsc{Scientific}: $1.82$ N@5 against $1.71$ for MCT-SASRec, $1.64$ for MCT-Router and $1.60$ for MCT-EnsCE.
The gain over single-hypothesis training is therefore obtained by the multiple-hypothesis objective alone.
Moreover, the curve is not monotonic: $\alpha=0.2$ and $\alpha=0.5$ fall below $\alpha=0$, so the effect of the score loss is of the same order as the variation across weights.
The score loss acts as an auxiliary regularizer of the shared encoder rather than as the source of the improvement.

\begin{figure}[!htbp]
    \centering
    \includegraphics[width=0.48\linewidth]{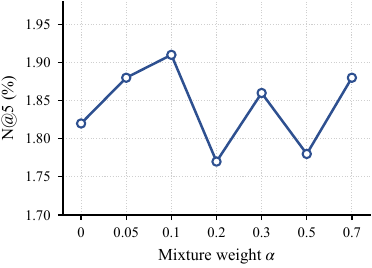}
    \caption{\textbf{Effect of the score-loss weight~$\alpha$.}
    N@5 (\%) on \textsc{Scientific} with $M=3$ under max-similarity aggregation, as a function of the weight~$\alpha$ of the score loss in \cref{eq:training-loss}.
    Ticks are evenly spaced and do not follow the scale of~$\alpha$.
    Without the score loss ($\alpha=0$), COMPOTE already reaches $1.82$, above every single-hypothesis variant of \cref{tab:results-ablation}; the best weight $\alpha=0.1$ adds $0.09$ points.}
    \label{fig:alpha-ablation}
\end{figure}

\section{The Confidence-Score Stabilization}
\label{appendix-section-describing-regularizer}

This appendix formalizes the regularization term mentioned at the end of \cref{sec:learning-problem}.
It is an auxiliary supervised objective attached to the confidence-score head $\hat{\gamma}$ of \cref{fig:mcl-transformer}, which is trained to predict which hypothesis the target item belongs to.

\paragraph{Confidence-score head.}
Recall the architecture of \cref{fig:mcl-transformer}: the shared causal encoder maps the history to the hidden state $\rvz = \mathrm{Enc}(\rvx) \in \lR^d$ at the most recent position, the $M$ hypothesis heads $h_1,\ldots,h_M$ map $\rvz$ to query representations $h_m(\rvx)\in\lR^d$, and the catalogue scores $\rvs_{m,i}(\rvx)$ of \cref{eq:max-similarity-aggregation} are turned into the per-hypothesis distributions $f_{m,\cdot}(\rvx) = \softmax_i\big(\rvs_{m,\cdot}(\rvx)\big) \in \Delta^{N\mhyphen1}$.
In parallel, an MLP taking the same $\rvz$ as input, followed by a coordinate-wise sigmoid, defines the \emph{confidence-score head}
\begin{align}
    \label{eq:confidence-head}
    \hat{\gamma} : \cX \to (0,1)^M\,,
\end{align}
where $\hat{\gamma}_m(\rvx)$ is read as the predicted probability that hypothesis $m$ is the one containing the target of query $\rvx$.
Note that the $M$ outputs are independent sigmoids and not a softmax: $\hat\gamma(\rvx)$ need not sum to one.

\paragraph{Supervision target: belonging to $\cY_m$.}
Following the tessellation of \cref{thm:closed-form}, the empirical cell of hypothesis $m$ on query $\rvx$ is the set of items for which $m$ is the most confident hypothesis,
\begin{align}
    \label{eq:empirical-cell}
    \cY_m(\rvx) = \Big\{ y \in \cY ~:~ \forall m'\neq m,~ -\log f_{m,y}(\rvx) < -\log f_{m',y}(\rvx) \Big\}\,,
\end{align}
so that $(\cY_m(\rvx))_{m\leq M}$ is a partition of $\cY$ up to ties.
The \emph{winning hypothesis} of a training pair $(\rvx,\rvy)$ is the index of the cell that contains the target,
\begin{align}
    \label{eq:winning-head}
    m^\star(\rvx,\rvy)
    = \argmin_{m\leq M} -\log f_{m,\rvy}(\rvx)
    \quad\text{i.e.}\quad
    \rvy \in \cY_{m^\star(\rvx,\rvy)}(\rvx)\,,
\end{align}
which is exactly the hypothesis realizing the $\min$ in $\ell^\textsc{wta}(\rvy, f(\rvx))$.
The regression target of $\hat{\gamma}$ is the corresponding membership vector
\begin{align}
    \label{eq:score-target}
    \gamma^\star(\rvx,\rvy) = \Big( \indicator{\rvy \in \cY_m(\rvx)} \Big)_{m=1}^{M} \in \{0,1\}^M\,,
\end{align}
which is a one-hot vector, with its non-zero coordinate at $m^\star(\rvx,\rvy)$.

\paragraph{Binary cross-entropy on $\hat{\gamma}$.}
The regularizer is the binary cross-entropy between the membership indicators \eqref{eq:score-target} and the predicted confidences, summed over the $M$ coordinates,
\begin{align}
    \label{eq:score-loss}
    \ell^{\rm score}(\rvx, \rvy, \hat{\gamma})
    = -\sum_{m=1}^{M}
    \Big[
        \indicator{\rvy \in \cY_m(\rvx)} \log \hat{\gamma}_m(\rvx)
        + \indicator{\rvy \notin \cY_m(\rvx)} \log\big(1 - \hat{\gamma}_m(\rvx)\big)
    \Big]\,,
\end{align}
which, the target being one-hot, can be written equivalently as
\begin{align}
    \ell^{\rm score}(\rvx, \rvy, \hat{\gamma})
    = -\log \hat{\gamma}_{m^\star(\rvx,\rvy)}(\rvx)
      - \sum_{m \neq m^\star(\rvx,\rvy)} \log\big(1 - \hat{\gamma}_m(\rvx)\big)\,.
\end{align}
The cells \eqref{eq:empirical-cell}, and hence the labels \eqref{eq:score-target}, depend on $f$: they are recomputed at every training step from the current hypotheses, and are detached from the computation graph, so that no gradient flows from $\ell^{\rm score}$ back to the heads $h_m$ through $m^\star$.
The only gradient path of $\ell^{\rm score}$ is through $\hat{\gamma}$, i.e. through the confidence MLP and the shared encoder producing $\rvz$.

\paragraph{Full training objective.}
Weighting this term by $\alpha > 0$, the objective minimized by COMPOTE is
\begin{align}
    \label{eq:training-loss}
    \cL_{\tau,\alpha}(f,\hat{\gamma})
    = \lE_{\rvx,\rvy}\Big[ \ell_\tau\big(\rvy, f(\rvx)\big) \Big]
      + \alpha\, \lE_{\rvx,\rvy}\Big[ \ell^{\rm score}(\rvx, \rvy, \hat{\gamma}) \Big]\,,
\end{align}
with $\ell_\tau$ the annealed loss of \cref{eq:loss-free-energy} -- equivalently $\ell^{\rm aMCL}_\tau$ of \cref{eq:awta} by \cref{prop:loss-tau-to-awta,cor:awta-gradient}.
We use $\alpha = 0.1$ in all experiments; the sensitivity to $\alpha$ is reported in \cref{app:alpha-ablation}.

\section{Impact of the number of hypotheses $M$ on learning}
\label{appendix-section-impact-of-M}

In  \cref{cor:wta-generalization-refined} in \cref{sec:learning-problem}, we have established
that the generalization  gap between the empirical WTA loss and the expected WTA loss has a sublinear
dependence on the number of hypotheses $M$. We present here an empirical validation of
this theoretical insight.

Keeping the experimental setup used for performance evaluation described in \cref{sec:empirical-evidence},
we vary the number of hypotheses $M$ from 1 to 12 and measure the ranking quality on \textsc{Scientific} using N@5.
\cref{fig:num-hypotheses} shows that increasing $M$ from 1 to 3 improves ranking quality,
which suggests that having multiple hypotheses is beneficial for capturing the ambiguity
of the next-item distribution. However, beyond $M=3$, the performance plateaus,
indicating that adding more hypotheses does not necessarily lead to better ranking quality.
This empirical finding  is consistent with the sublinear dependence on~$M$ in \cref{cor:wta-generalization-refined},
and provides practical guidance for selecting an appropriate number of hypotheses in real-world applications.

For those larger value of $M$, a phenomenom we have described in
\cref{prop:emax-ranking} in \cref{sec:compote-to-ranking} can also occur,
which is the saturation of the max-similarity aggregation $\omega$.
\cref{sec:compote-to-ranking}.
We have stated that this saturation can happen when at least two items
have large probabilities $\rvq_a^\star, \rvq_b^\star \geq 1/M$.
This means that as $M$ increases, the probability of this saturation event increases,
which can lead to a degradation of the ranking quality.

 \begin{figure}[h]
    \centering
    \includegraphics[width=0.48\linewidth]{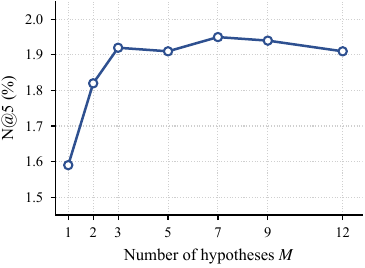}
    \caption{\textbf{Effect of the number of hypotheses.}
    N@5 on \textsc{Scientific} under max-similarity aggregation, as a function of~$M$.
    Performance improves sharply from $M=1$ to $M=3$, then plateaus: increasing $M$ from 3 to 12 does not degrade ranking quality, in line with the sublinear dependence on~$M$ in \cref{cor:wta-generalization-refined}.}
    \label{fig:num-hypotheses}
\end{figure}

\section{Collapse and Oracle Evidence}
\label{app:collapse-evidence}

\begin{table}[t]
    \caption{Mean pairwise cosine similarity between hypothesis heads on the test set.
    A value of $1$ means the heads are indistinguishable, lower is more specialized.}
    \label{tab:head-collapse}
    \centering
    \small
    \begin{tabular}{lccc}
        \toprule
        Method & Instrument & Scientific & Game \\
        \midrule
        MCT-EnsCE & 1.00 & 1.00 & 1.00 \\
        MCT-Router & 1.00 & 1.00 & 1.00 \\
        \midrule
        COMPOTE (ours) & \textbf{0.64} & \textbf{0.61} & \textbf{0.58} \\
        \bottomrule
    \end{tabular}
\end{table}

\begin{table}[h]
    \caption{NDCG@5 (\%) under max-similarity aggregation (\cref{eq:max-similarity-aggregation}) and under oracle head selection (\cref{eq:oracle-head}).
    Equal columns indicate that all heads induce the same ranking, i.e.\ complete collapse.}
    \label{tab:oracle-gap}
    \centering
    \small
    \setlength{\tabcolsep}{5pt}
    \begin{tabular}{lcccccc}
        \toprule
        & \multicolumn{2}{c}{Instrument} & \multicolumn{2}{c}{Scientific} & \multicolumn{2}{c}{Game} \\
        \cmidrule(lr){2-3} \cmidrule(lr){4-5} \cmidrule(lr){6-7}
        Method & Aggregated & Oracle & Aggregated & Oracle & Aggregated & Oracle \\
        \midrule
        MCT-EnsCE & 2.28 & 2.28 & 1.60 & 1.60 & 3.78 & 3.78 \\
        MCT-Router & 2.28 & 2.28 & 1.64 & 1.64 & 3.78 & 3.78 \\
        \midrule
        COMPOTE (ours) & \textbf{2.45} & \textbf{3.75} & \textbf{1.92} & \textbf{3.52} & \textbf{4.01} & \textbf{7.78} \\
        \bottomrule
    \end{tabular}
\end{table}

\section{Proofs of \cref{sec:hypothesis-specialization}}

\subsection{Proof of \cref{thm:closed-form}}
\label{ssec:thm:proof:closed-form}

\thmClosedForm*
\begin{proof}[Proof of \cref{thm:closed-form}]
    The characterization of $\rvp^\star$ using $Q_m$ and the tessalation $\cY_m$ are a direct application of \cref{thm:closed-form}.

    The only part remaining is to prove the tessalation $\cY_m$ is entropy-maximizing. Given a tessalation $C= \{C_1,\dots, C_m\}$ of $\cY$, define
    \begin{align}
        \rvp^C_{m,y} = \frac{\rvq^\star_y}{Q^C_m}\indicator{y\in C_m}
        \quad\text{where}\quad
        Q^C_m = \sum_{y\in C_m} \rvq^\star_y
    \end{align}
    \begin{align}
        \lE_{\rvy\sim\rvq^\star}[\ell^\textsc{wta}(\rvy, \rvp^C)]
        & = -\sum_{m=1}^M \sum_{y\in C_m} \rvq^\star_y \log(\rvp^C_{m,y})\\
        & = -\sum_{m=1}^M \sum_{y\in C_m} \rvq^\star_y \log\left(\frac{\rvq^\star_y}{Q^C_m}\right)\\
        & = H(\rvq^\star) - H(Q^C)
    \end{align}
    The choice $C^\star = (\cY_m)_{m=1}^M$ coincides with the optimal predictor -- i.e. $\rvp^{C^\star} = \rvp^\star$, leading to
    \begin{align}
        \lE_{\rvy\sim\rvq^\star}&[\ell^\textsc{wta}(\rvy, \rvp^{C^\star})] = \min_{C:\text{M-tessalation of}\cY}\!\!\!\! \lE_{\rvy\sim\rvq^\star}[\ell^\textsc{wta}(\rvy, \rvp^C)]\\
        &\Leftrightarrow H(\rvq^\star) - H(Q^{C^\star}) = \min_{C:\text{M-tessalation of}\cY}\!\!\!\! H(\rvq^\star) - H(Q^C)\\
        &\Leftrightarrow H(Q^{C^\star}) = \max_{C:\text{M-tessalation of}\cY}\!\!\!\! H(Q^C)
    \end{align}
\end{proof}

\subsection{Optimal Risk}
\label{ssec:cor:closed-form-risk}

The following corollary makes explicit the value of the optimal winner-takes-all risk.

\begin{restatable}[Optimal winner-takes-all risk]{corollary}{corClosedFormRisk}
    \label{cor:closed-form-risk}
    Under the setting of \cref{thm:closed-form}, with $Q=(Q_m)_{m=1}^M$ the cell masses of the tessalation $(\cY_m)_{m=1}^M$,
    \begin{align}
        \label{eq:closed-form-risk}
        \min_{\rvp\in\left(\Delta^{N\mhyphen1}\right)^M} \lE_{\rvy\sim\rvq^\star}[\ell^\textsc{wta}(\rvy, \rvp)]
        = H(\rvq^\star) - H(Q)
        \geq H(\rvq^\star) - \log M\,,
    \end{align}
    with equality in the inequality if and only if $Q_m = 1/M$ for every $m\leq M$.
\end{restatable}

\begin{proof}[Proof of \cref{cor:closed-form-risk}]
    By \cref{thm:closed-form}, $\rvp^\star = \rvp^{C^\star}$ with $C^\star = (\cY_m)_{m=1}^M$, and the computation in the proof of \cref{thm:closed-form} gives
    \begin{align}
        \min_{\rvp\in\left(\Delta^{N\mhyphen1}\right)^M} \lE_{\rvy\sim\rvq^\star}[\ell^\textsc{wta}(\rvy, \rvp)]
        = \lE_{\rvy\sim\rvq^\star}[\ell^\textsc{wta}(\rvy, \rvp^{C^\star})]
        = H(\rvq^\star) - H(Q)\,.
    \end{align}
    Since $Q\in\Delta^{M\mhyphen1}$, $H(Q)\leq\log M$, with equality if and only if $Q$ is uniform.
\end{proof}

\subsection{The Loss Follows the $\log M$ Law}
\label{ssec:log-m-law}

\Cref{thm:closed-form} also predicts how the loss should change with $M$.
For a given query, the optimal winner-takes-all loss is the Bayes loss minus the entropy $H(Q(\rvx))$ of the cell masses (\cref{cor:closed-form-risk}).
Exact querywise balance would therefore yield a $\log M$ reduction relative to the Bayes loss.
Although the win frequencies of \cref{sec:what-hypotheses-learn} establish only marginal balance, empirically $\hat\cL^\textsc{wta}+\log M$ varies by less than $0.1$ across $M\in\{3,4,5\}$ (\cref{tab:latent-tiling-theory}), showing that the learned models approximately follow this scaling law, where $\hat\cL^\textsc{wta}=\lE[\min_m -\log f_{m,\rvy}(\rvx)]$ is the test loss.
Equivalently, going from $M=3$ to $M=5$ multiplies the geometric-mean probability that the winning head assigns to the target by $e^{0.61}\approx1.83$, close to and slightly above the ratio $5/3\approx1.67$ predicted under exact balance.

\section{Latent Tiling Across $M$}
\label{app:latent-tiling}

This appendix details the protocol behind \cref{tab:latent-tiling-theory}, reports additional diagnostics on the three COMPOTE checkpoints trained on \textsc{Game} with $M\in\{3,4,5\}$, and checks that the partition also appears on \textsc{Scientific}.

\paragraph{Protocol.}
Win frequencies, static sectors, and sector agreement use the full test set of $94{,}762$ query--target pairs.
In-cell mass and the winner-takes-all loss use the softmax $f_m(\rvx)$ over the whole catalogue, with $\beta=0.07$, on the same full test set.
The in-cell mass is $\sum_{y\in\cY_m(\rvx)} f_{m,y}(\rvx)$, averaged over heads and test pairs.
The static direction $\widetilde{\rvh}_m$ of head $m$ is the renormalized average of its normalized outputs over the test set.
Writing $\bar{\rve}_i=\rve_i/\|\rve_i\|_2$ for the normalized embedding of item $i$, its sector is $\argmax_m \langle \bar{\rve}_i,\widetilde{\rvh}_m\rangle$.
Sector agreement is the fraction of test targets whose sector is their winning head.
Cross-head Jaccard, region loyalty, and single-head reach are computed from the top-$10$ lists on $4{,}000$ test queries sampled uniformly.
Among catalogue items retrieved at least once, an item is reached by a single head if it appears in the top-$10$ of exactly one head.
Within-head overlap compares the top-$10$ list of a head between randomly paired queries from this sample.

\paragraph{Cosine and softmax winners agree.}
The oracle of \cref{eq:oracle-head} selects heads by cosine similarity, whereas the cells of \cref{thm:closed-form} are defined by probabilities, and each head carries its own softmax normalizer.
The two definitions of the winning head agree on $98.9\%$, $99.4\%$, and $99.7\%$ of test pairs for $M=3,4,5$.

\paragraph{The tiling sharpens with $M$.}
The static directions are well separated for every $M$ (\cref{fig:latent-tiling-ablation}).
Cross-head top-$10$ Jaccard falls from $0.008$ at $M=3$ to $0.001$ at $M=5$, while region loyalty rises from $98.9\%$ to $99.8\%$ (\cref{tab:latent-tiling-ablation}).
\Cref{thm:closed-form} predicts an exact partition for every $M$, so this trend describes how close training gets to the optimum rather than a property of the optimum itself.

\paragraph{The same partition appears on \textsc{Scientific}.}
We repeat the overlap diagnostics on $4{,}000$ sampled test queries for the COMPOTE checkpoint trained on \textsc{Scientific} with $M=3$.
The top-$10$ lists of two heads on the same query have a mean Jaccard overlap of $0.001$, and $94.5\%$ of the catalogue items retrieved at test time are reached by a single head only.
The top-$10$ of a given head overlaps by $0.025$ to $0.044$ Jaccard between two different histories, so each head still adapts to the query.
\Cref{fig:data-visualisation} shows three test queries.
A plain PCA of the normalized catalogue already reveals three arms, and the heads sit along distinct arms, with the target close to one of them.

\begin{table}[t]
    \caption{Latent-space tiling on \textsc{Game} as the number of hypotheses $M$ varies.
    Mean off-diagonal cosine similarity is computed between the static head directions $\widetilde{\rvh}_m$ (the dataset-wide average of normalized head outputs).
    The top-$10$ diagnostics use $4{,}000$ test queries sampled uniformly without replacement.
    Cross-head Jaccard is the mean pairwise overlap of these lists across heads.
    Region loyalty is the fraction of a head's top-$10$ items that fall in its catalogue sector.}
    \label{tab:latent-tiling-ablation}
    \centering
    \small
    \begin{tabular}{cccc}
        \toprule
        $M$ & Mean off-diag.\ $\cos(\widetilde{\rvh})$ & Cross-head top-$10$ Jaccard & Region loyalty \\
        \midrule
        3 & $-0.09$ & $0.008$ & $98.9\%$ \\
        4 & $-0.17$ & $0.003$ & $99.4\%$ \\
        5 & $-0.04$ & $0.001$ & $99.8\%$ \\
        \bottomrule
    \end{tabular}
\end{table}

\begin{figure}[!htbp]
    \centering
    \begin{subfigure}[t]{0.32\textwidth}
        \centering
        \includegraphics[width=\linewidth]{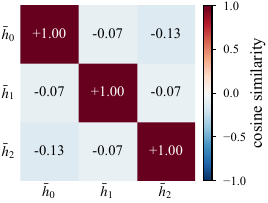}
        \caption{$M=3$}
        \label{fig:latent-tiling-game-3}
    \end{subfigure}\hfill
    \begin{subfigure}[t]{0.32\textwidth}
        \centering
        \includegraphics[width=\linewidth]{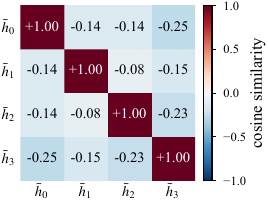}
        \caption{$M=4$}
        \label{fig:latent-tiling-game-4}
    \end{subfigure}\hfill
    \begin{subfigure}[t]{0.32\textwidth}
        \centering
        \includegraphics[width=\linewidth]{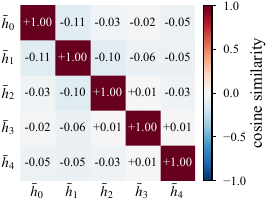}
        \caption{$M=5$}
        \label{fig:latent-tiling-game-5}
    \end{subfigure}
    \caption{\textbf{Latent-space tiling on \textsc{Game}.}
    Cosine similarity between the static head directions $\bar{\rvh}_m$ (the dataset-wide average of normalized head outputs) for $M\in\{3,4,5\}$.
    Off-diagonal values remain near zero or negative, indicating that the heads occupy separated directions rather than collapsing onto one another.}
    \label{fig:latent-tiling-ablation}
\end{figure}

\begin{figure}[!htbp]
    \centering
    \begin{subfigure}[t]{0.32\textwidth}
        \centering
        \includegraphics[width=\linewidth]{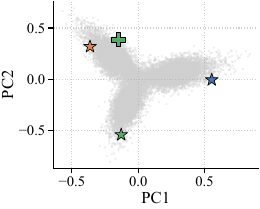}
        \label{fig:data-visualisation-4934}
    \end{subfigure}\hfill
    \begin{subfigure}[t]{0.32\textwidth}
        \centering
        \includegraphics[width=\linewidth]{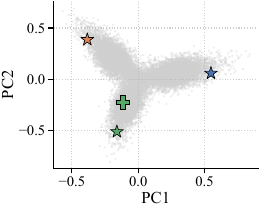}
        \label{fig:data-visualisation-23025}
    \end{subfigure}\hfill
    \begin{subfigure}[t]{0.32\textwidth}
        \centering
        \includegraphics[width=\linewidth]{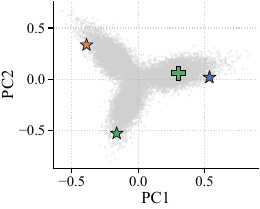}
        \label{fig:data-visualisation-26314}
    \end{subfigure}
    \caption{\textbf{Catalogue partition on \textsc{Scientific} ($M=3$).} PCA projections on three test sequences (queries~4934, 23025, and 26314) under COMPOTE.
    Each panel shows the L2-normalized item catalogue in two dimensions (grey), the target next item (green cross), and the three head query embeddings (stars).
    The heads occupy distinct regions of the catalogue rather than collapsing onto a single direction.}
    \label{fig:data-visualisation}
\end{figure}

\section{Beyond the Theorem: a Static Tiling}
\label{app:static-tiling}

\Cref{thm:closed-form} allows the cells to depend on the query.
To test whether they do, we average the normalized outputs of each head over the test set and assign every catalogue item to the static sector of the direction it is most aligned with (protocol in \cref{app:latent-tiling}).
The sector of the target predicts its winning head on $98.5\%$ to $99.4\%$ of test pairs (last column of \cref{tab:latent-tiling-theory}), and more than $98.9\%$ of each head's top-$10$ items lie in its sector.
Cell membership is thus almost a function of the item alone.
COMPOTE learns a fixed tiling of the catalogue, and only the ranking inside each tile adapts to the history, as (P2) requires.
A static tiling cannot balance every query: for a user whose likely next items all lie in one tile, $Q(\rvx)$ is concentrated on a single head, and that head gets no boost from the partition.
The balance reported in \cref{sec:what-hypotheses-learn} holds on average over queries, which is what win frequencies measure.





\end{document}